\documentclass{article}

\usepackage{microtype}
\usepackage{graphicx}
\usepackage{subcaption}
\usepackage{booktabs} 
\usepackage{colortbl}
\usepackage{xcolor}
 \usepackage{multirow} 

\usepackage{hyperref}

\usepackage[preprint]{icml2026}

\usepackage{amsmath}
\usepackage{amssymb}
\usepackage{mathtools}
\usepackage{amsthm}
\usepackage{enumitem}

\usepackage{algorithm}
\usepackage{algorithmic}

\usepackage[capitalize,noabbrev]{cleveref}

\theoremstyle{plain}
\newtheorem{theorem}{Theorem}[section]
\newtheorem{proposition}[theorem]{Proposition}
\newtheorem{lemma}[theorem]{Lemma}

\theoremstyle{definition}

\newtheorem{assumption}[theorem]{Assumption}
\theoremstyle{remark}

\usepackage[textsize=tiny]{todonotes}

\icmltitlerunning{IRIS: Implicit Reward-Guided Internal Sifting
for Mitigating Multimodal Hallucination}

\begin{document}

\twocolumn[
  \icmltitle{IRIS: Implicit Reward-Guided Internal Sifting \texorpdfstring{\\}{ }
    for Mitigating Multimodal Hallucination}



\icmlsetsymbol{equal}{*}
\icmlsetsymbol{corr}{*}
   \begin{icmlauthorlist}
    \icmlauthor{Yuanshuai Li}{wu}
    \icmlauthor{Yuping Yan}{wu}
    \icmlauthor{Jirui Han}{wu}
    \icmlauthor{Fei Ming}{wu}
    \icmlauthor{Lingjuan Lv}{sony}
    \icmlauthor{Yaochu Jin}{wu,corr}
  \end{icmlauthorlist}

  \icmlaffiliation{wu}{Department of Artificial Intelligence, Westlake University, Hangzhou, China}
  \icmlaffiliation{sony}{Sony Research, Sony}

  \icmlcorrespondingauthor{Yaochu Jin}{jinyaochu@westlake.edu.cn}

  \icmlkeywords{Multimodal Hallucination, DPO, ICML}

  \vskip 0.3in
]



\printAffiliationsAndNotice{}  

\begin{abstract}
Hallucination remains a fundamental challenge for Multimodal Large Language Models (MLLMs). While Direct Preference Optimization (DPO) is a key alignment framework, existing approaches often rely heavily on costly external evaluators for scoring or rewriting, incurring off-policy learnability gaps and discretization loss. Due to the lack of access to internal states, such feedback overlooks the fine-grained conflicts between different modalities that lead to hallucinations during generation.
To address this issue, we propose \textbf{IRIS} (\textbf{I}mplicit \textbf{R}eward-Guided \textbf{I}nternal \textbf{S}ifting), which leverages continuous implicit rewards in the native log-probability space to preserve full information density and capture internal modal competition. This on-policy paradigm eliminates learnability gaps by utilizing self-generated preference pairs. By sifting these pairs based on multimodal implicit rewards, IRIS ensures that optimization is driven by signals that directly resolve modal conflicts.
Extensive experiments demonstrate that IRIS achieves highly competitive performance on key hallucination benchmarks using only 5.7k samples, without requiring any external feedback during preference alignment. These results confirm that IRIS provides an efficient and principled paradigm for mitigating MLLM hallucinations.

\end{abstract}

\section{Introduction}
\label{sec:intro}

Multimodal Large Language Models (MLLMs)~\citep{achiam2023gpt,liu2023visual} integrate pretrained visual encoders with Large Language Models (LLMs) to achieve strong performance on vision-language tasks. However, these models often suffer from hallucinations, where the generated text contradicts the provided visual evidence~\citep{liu2024survey}.
Fundamentally, hallucinations arise from an imbalance between modalities during the generation process. MLLMs exhibit a strong dependence on statistical language priors acquired from large-scale textual pretraining~\citep{leng2025the}. These priors can dominate the influence of visual signals during generation, leading the model to prioritize linguistically plausible responses that are insufficiently grounded in visual evidence~\cite{leng2024mitigating}. Consequently, the model fails to faithfully ground its generation in visual inputs, especially when visual information contradicts common linguistic patterns.

\begin{figure}[t]
    \centering
    \includegraphics[width=\linewidth]{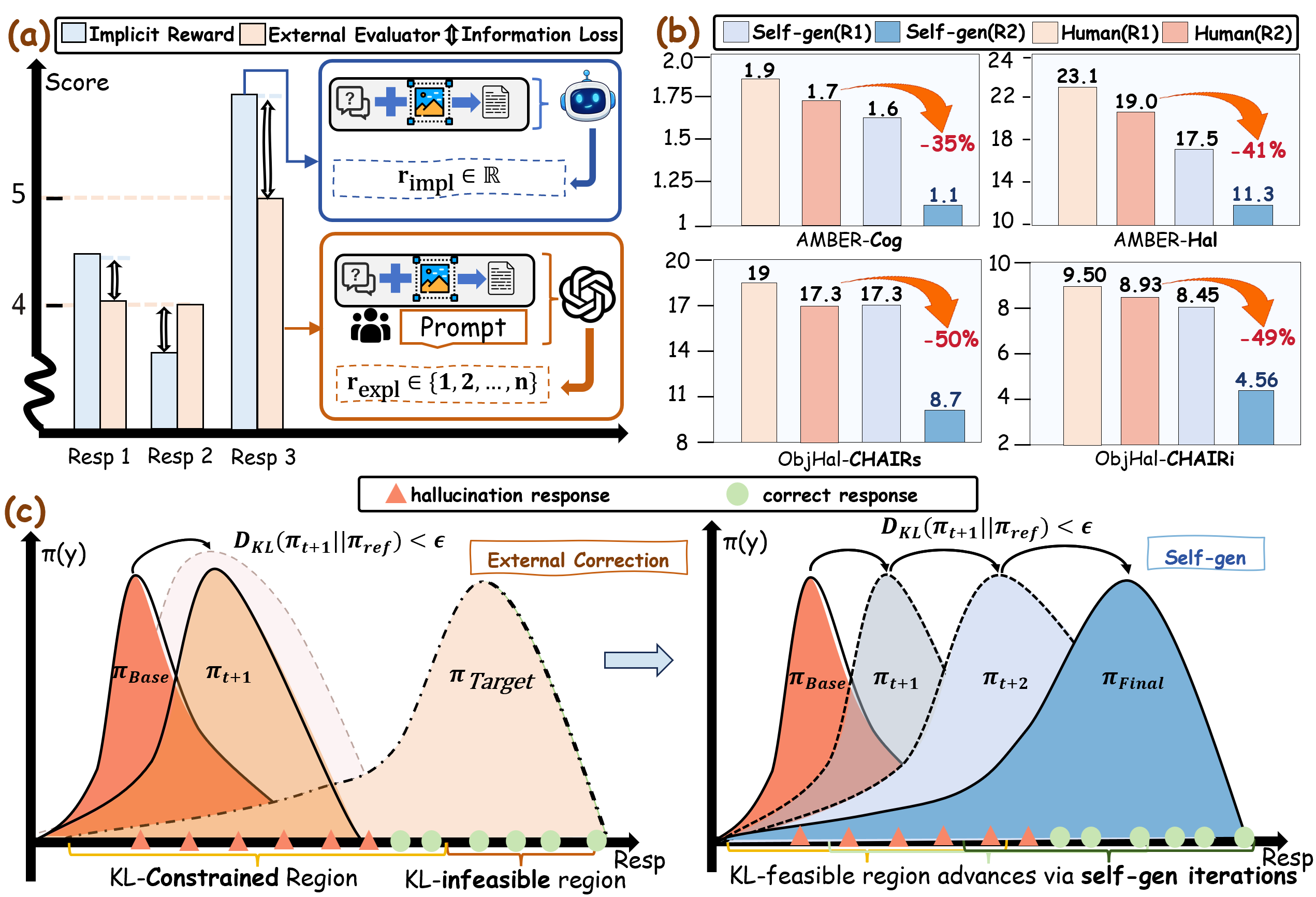} 
\caption{
\textbf{Overview of IRIS.} 
\textbf{(a)} Comparison between external discrete rewards and IRIS implicit rewards. 
\textbf{(b)} Empirical hallucination reduction achieved by self-generated preference optimization compared to human-annotated data under the same experimental setting. 
\textbf{(c)} Conceptual illustration of policy evolution under KL constraints: external correction fails to cross the KL-infeasible region, while IRIS progressively advances the KL-feasible region via self-generated iterations.}
    \label{fig:fig1}
    \vskip -0.2in
\end{figure}

To better align MLLMs with visual evidence, Direct Preference Optimization (DPO)~\cite{rafailov2023direct} offers a rigorous objective for preference alignment by directly optimizing the policy on preference pairs without explicit reward modeling. This objective enables stable optimization under KL-divergence constraints~\cite{kullback1951information} by establishing a mapping between the reward function and the optimal policy. Despite its widespread adoption, the efficacy of DPO in mitigating multimodal hallucinations is highly dependent on whether preference signals can capture the degree of visual grounding in the model’s own generative process.

This leads to a fundamental question: which preference signals can reliably quantify the degree of visual grounding during the generation process?
Existing approaches predominantly utilize external evaluators, such as GPT-4V, to provide discrete scores or corrective feedback~\cite{yu2024rlaif,yang2025mitigating,liu2025mitigating,li2025mitigatingvisualhallucinationssemantic}. However, these signals primarily assess output semantics and fail to characterize the specific internal mechanisms responsible for hallucinations along the model's own generative trajectories, leading to two key limitations.
\textbf{First,} external evaluators induce information loss through discretization. They compress the model's continuous probability distribution into a restricted set of labels. As illustrated in Fig.~\ref{fig:fig1}a, discrete external rewards assign identical scores to semantically distinct responses,
collapsing fine-grained preference differences that are essential for accurate visual grounding~\cite{wang2025improving}.
\textbf{Second,} external preference signals introduce a structural distributional discrepancy that hinders optimization. The reverse KL-divergence constraint in the DPO objective confines policy refinement to the support of the reference distribution. When preferred responses originate from disjoint off-policy distributions, they receive vanishing probability under the reference policy, so the corresponding log-ratio terms contribute almost no effective gradient and the update weights approach zero, rendering the feedback unlearnable~\cite{guo2024direct}, as illustrated by the target lying beyond the KL-constrained region in Fig.~\ref{fig:fig1}c (left).

In contrast to external evaluators, DPO leverages an \textit{implicit reward} that is defined from the model’s own policy~\cite{rafailov2023direct}. It is the log-likelihood ratio between the current and reference policies, and it provides a continuous signal in log-probability space. Compared with discrete feedback, it retains fine-grained preference differences that discretization would discard. Furthermore, as an on-policy signal, it avoids distribution shift. This ensures stable learning under KL constraints and helps identify when language priors override visual evidence. After an SFT warm-up for visual consistency~\cite{wang2025implicit}, the implicit reward becomes a reliable signal for constructing and optimizing preference pairs. This supports iterative on-policy refinement that improves visual grounding and reduces hallucinations.

Inspired by these insights, we introduce \textbf{IRIS} (\textbf{I}mplicit \textbf{R}eward-Guided \textbf{I}nternal \textbf{S}ifting), a multimodal preference alignment framework that leverages intrinsic implicit rewards as the primary alignment signal. 
By eliminating the dependency on external evaluators, IRIS enables the model to autonomously refine its policy using its native implicit rewards (Fig.~\ref{fig:fig1}a). 
The framework operates in two stages. First, a preliminary SFT phase is conducted for value calibration, anchoring the model’s latent distribution to visual consistency. 
Building on this foundation, the model generates candidate responses under its native policy and utilizes its own implicit rewards to construct targeted on-policy preference pairs via the proposed Rectified Visual Guidance (RVG) scoring. 
These pairs are subsequently optimized via multimodal DPO in an iterative refinement cycle, ensuring that the alignment remains grounded in the model’s intrinsic generative distribution (Fig.~\ref{fig:fig1}c).

Our primary contributions are summarized as follows:
\begin{itemize}[noitemsep, topsep=0pt, parsep=0pt, partopsep=0pt]
    \item We identify the limitations of discrete external feedback in multimodal DPO, and show that implicit rewards provide a continuous signal better suited for mitigating hallucinations.
    \item We propose \textbf{IRIS}, an efficient and principled paradigm that leverages Rectified Visual Guidance (RVG) scoring to sift on-policy preference pairs, ensuring the alignment is grounded in the model's native distribution.
    \item With only 5.7k samples and no external feedback during alignment, IRIS achieves strong and competitive performance across multiple benchmarks, matching or outperforming baselines trained with substantially larger datasets and external evaluators on key hallucination metrics.

\end{itemize}


\section{Related Work}
\label{sec:related}

\subsection{Hallucination Mitigation in MLLMs}
\label{sec:related_hallucination}

Research on mitigating hallucinations in MLLMs has evolved from inference-time decoding strategies~\citep{leng2024mitigatingVCD,chen2024halc} toward training-time preference alignment. Recently, DPO-based approaches have become the prevailing paradigm for enhancing visual grounding, primarily due to their superior stability and computational efficiency compared to traditional Reinforcement Learning from Human Feedback (RLHF) frameworks~\citep{yu2024rlhf} that rely on complex and often unstable optimization procedures like PPO~\citep{schulman2017proximal}.

The efficacy of DPO alignment depends on the quality of preference pairs, which provide the essential supervision to distinguish grounded responses from hallucinations. Early research focused on ranking-based automated exploration. These methods establish heuristic rules to construct preference pairs, employing heuristic metrics such as cross-modal similarity~\citep{ouali2024clip}, model scaling priors~\citep{zhang2024automated}, or visual input perturbations~\citep{pi2024strengthening} to estimate response quality.

To achieve higher alignment precision, the focus has shifted toward expert-led feedback. This progression has moved from fine-grained human annotations~\citep{yu2024rlhf} to leveraging proprietary models like GPT-4V~\citep{yang2025mitigating,liu2025mitigating}, and more recently, to utilizing powerful open-source models~\citep{yu2024rlaif,liu2025mitigating} as external evaluators for hallucination detection and rewriting.

Despite these differences, these approaches rely on an external supervision paradigm and typically utilize off-policy data. Consequently, they are limited by the capabilities of the evaluators and fail to address the internal causes of hallucinations during the model's own generation process.

\subsection{Self-Alignment via DPO Implicit Rewards}
\label{sec:related_self_alignment}

Although DPO and its variants are widely adopted for their simplicity, their offline nature can induce distribution shift, limiting policy improvement and potentially leading to overfitting~\citep{guo2024direct}. Prior work suggests that incorporating on-policy sampling to provide dynamic feedback significantly enhances alignment stability and performance~\citep{tajwar2024preference}. Consequently, mining and filtering self-generated samples for self-alignment has emerged as a key strategy to overcome the inherent limitations of offline DPO.

Theoretical evidence supports this internal evaluation approach. \citet{rafailov2024r} showed that DPO-trained models implicitly define a dense reward function at the token level. Recent research further confirms that models have the potential to evaluate themselves. For instance, it is found in \citet{li2025generalist} that models possess internal rewards for self-evaluation, while authors in \citet{wang2025implicit} showed that SFT helps calibrate these reward signals. Based on these findings, textual alignment methods such as DICE~\citep{chen2024bootstrapping} and SeRA~\citep{ko2025sera} have been proposed to leverage implicit reward signals for general quality improvement through sample bootstrapping and filtering.

However, the potential of implicit rewards for multimodal hallucination mitigation remains unexplored. We argue that this signal is naturally suited for this task because it directly reflects the internal competition between visual evidence and language priors. This allows us to detect hallucinations within the model's own probability space, which is not possible with external evaluators that only observe final outputs.

\section{Preliminaries}
\label{sec:preliminaries}

We formalize the problem of multimodal hallucination mitigation within the framework of preference optimization. Let $v$ denote the visual input (image), $x$ the textual instruction, and $y = (y_1, \dots, y_T)$ the generated response sequence.
\paragraph{Supervised Fine-Tuning (SFT)}

SFT is a widely adopted technique to adapt pre-trained MLLMs to downstream tasks. Given a dataset $\mathcal{D}_{\text{SFT}} = \{(v, x, y)\}$ comprising visual inputs $v$, textual prompts $x$, and corresponding ground-truth responses $y$, the training objective is to maximize the negative likelihood of the target response in an auto-regressive manner:
\begin{equation}
    \mathcal{L}_{\text{SFT}}(\pi_\theta) = -\mathbb{E}_{(v, x, y)} \left[ \sum_{t=1}^{|y|} \log \pi_\theta(y_t \mid v, x, y_{<t}) \right],
\end{equation}
where $y_t$ denotes the $t$-th token in the target sequence, $y_{<t}$ represents the sequence of preceding tokens, and $\pi_\theta(y_t \mid v, x, y_{<t})$ is the conditional probability predicted by the model.

\paragraph{Direct Preference Optimization (DPO).}

DPO optimizes a policy model using paired preference data $\mathcal{D}_{\text{pref}} = \{(v, x, y_w, y_l)\}$, where $y_w$ is preferred over $y_l$ given the same input $(v, x)$. Under the Bradley-Terry (BT) model assumption, the preference probability is determined by a latent reward function $r^*$:
\begin{equation}
    P(y_w \succ y_l \mid v, x) = \sigma(r^*(v, x, y_w) - r^*(v, x, y_l)),
\end{equation}
where $\sigma$ denotes the sigmoid function. DPO derives a closed-form mapping between the optimal reward function $r^*$ and the optimal policy $\pi^*$. Specifically, the implicit reward is expressed as:
\begin{equation}
    r^*(v, x, y) = \beta \log \frac{\pi^*(y \mid v, x)}{\pi_{\text{ref}}(y \mid v, x)} + Z(v, x),
    \label{eq:implicit_reward}
\end{equation}
where $\beta$ controls the deviation from the reference policy and $Z(v, x)$ is a partition function. Based on this formulation, the policy $\pi_\theta$ is optimized by minimizing the following objective:
\begin{equation}
\begin{aligned}
    \mathcal{L}_{\text{DPO}}(\pi_\theta) = - \mathbb{E}_{(v, x, y_w, y_l)} \Bigg[ \log \sigma \Bigg( 
    &\beta \log \frac{\pi_\theta(y_w \mid v, x)}{\pi_{\text{ref}}(y_w \mid v, x)} \\
    - \, &\beta \log \frac{\pi_\theta(y_l \mid v, x)}{\pi_{\text{ref}}(y_l \mid v, x)} \Bigg) \Bigg].
\end{aligned}
\label{eq:dpo_loss}
\end{equation}
This formulation unifies reward modeling and policy optimization into a single objective, eliminating the need for an explicit reward model typically required in reinforcement learning.

\begin{figure}[t]
    \centering
    \includegraphics[width=\linewidth]{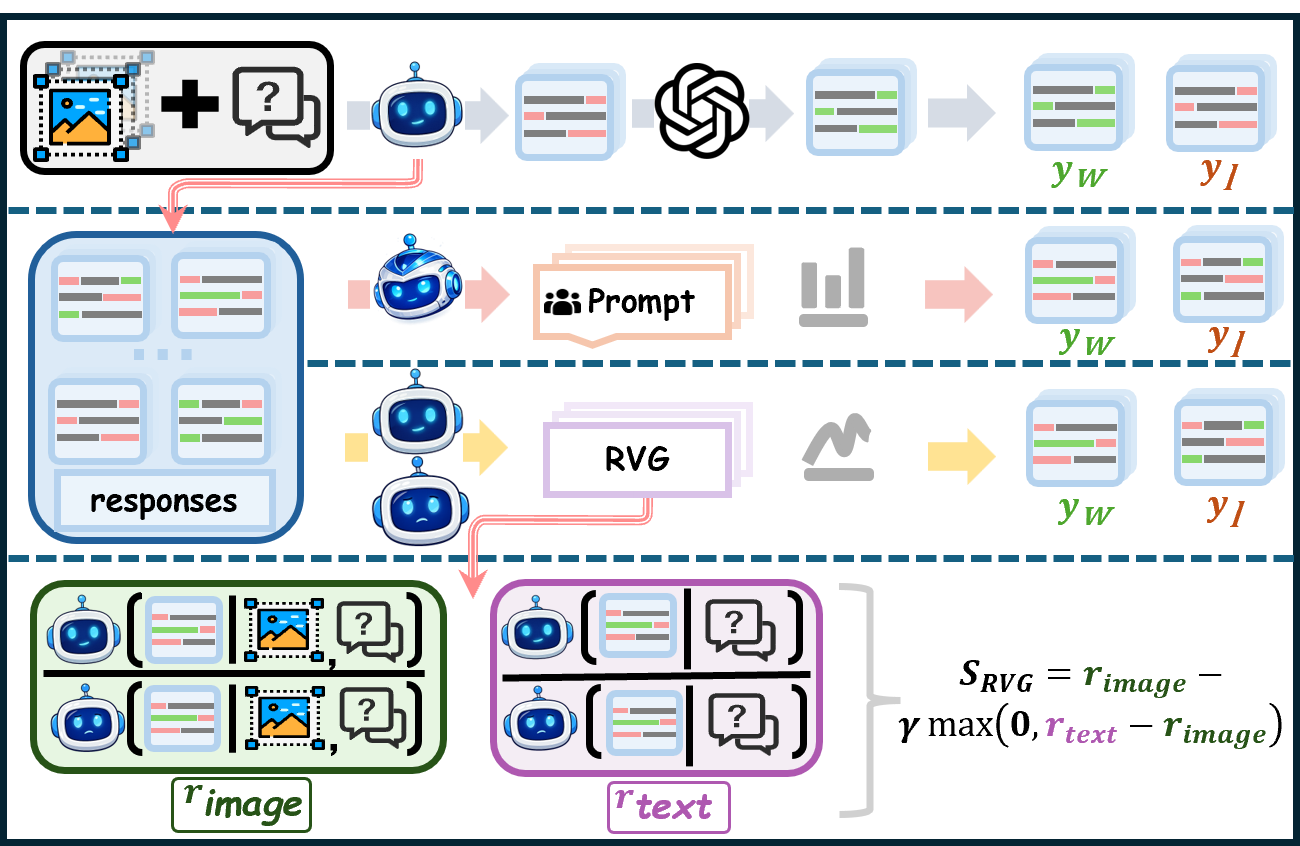} 
\caption{\textbf{Comparison of Preference Construction Pipelines.} 
\textbf{Top:} Feedback from proprietary models (e.g., GPT-4). 
\textbf{Middle:} Prompt-based scoring using large open-source models. 
\textbf{Bottom:} \texttt{IRIS} (Ours), which leverages intrinsic implicit rewards and Rectified Visual Guidance (RVG) to sift on-policy preference pairs without external evaluators.}
    \label{fig:fig2}
     \vskip -0.1in
\end{figure}

\section{Methodology}
\label{sec:method}

\subsection{Theoretical Motivation}

\paragraph{Learnability of Noisy Self-Generated Preferences.}
A potential concern is that self-generated candidates $(y_w,y_l)$ may be imperfect and even share hallucinated content.
Preference optimization, however, depends on \emph{relative} comparisons rather than absolute correctness.
Under the delta learning view~\citep{geng2025delta}, learning can still progress as long as the induced preference direction is correct more often than not.

For a generic pairwise objective, the gradient admits a difference form:
\begin{equation}
\begin{aligned}
\nabla_{\theta}\mathcal{L}_{\text{pref}}
&=
-\,w(\cdot)\Big(
\nabla_{\theta}\log \pi_{\theta}(y_w\mid v,x) \\
&\qquad\qquad
-
\nabla_{\theta}\log \pi_{\theta}(y_l\mid v,x)
\Big),
\end{aligned}
\label{eq:pairwise_grad_delta}
\end{equation}
where $w(\cdot)>0$ is a scalar weight determined by the loss.
Thus, the update increases the log-likelihood gap
$\log \pi_{\theta}(y_w\mid v,x)-\log \pi_{\theta}(y_l\mid v,x)$
in expectation over sampled preference pairs, pushing the policy toward responses that are preferred under the same context.
A formal derivation is provided in Appendix~\ref{app:delta_learning}.

\paragraph{Implicit Reward Calibration via SFT.}
Self-alignment with implicit rewards requires the model to rank grounded responses above hallucinated ones.
We posit that this capability is established during the SFT warm-up, which calibrates the model’s scoring toward task-relevant visual grounding.
Recent analyses relate maximum-likelihood training to implicit reward learning under KL-regularized distribution matching~\citep{wang2025implicit}.

In particular, under an implicit-reward formulation, the soft-optimal policy admits
\begin{equation}
\log \pi^{*}(a\mid s)
=
\log \pi_{\text{ref}}(a\mid s)
+\frac{1}{\beta}Q^{*}(s,a)
- V^{*}(s),
\label{eq:logpi_q}
\end{equation}
and hence
$\log \frac{\pi^{*}(a\mid s)}{\pi_{\text{ref}}(a\mid s)}
=
\frac{1}{\beta}Q^{*}(s,a) - V^{*}(s)$.
Motivated by this connection, we use the DPO implicit reward
$r(y\mid s)=\beta\log\frac{\pi(y\mid s)}{\pi_{\text{ref}}(y\mid s)}$
as an intrinsic scoring signal.

We treat the generation context as the state: $s=(v,x)$ for image-conditioned generation and $s=(v_{\emptyset},x)$ for text-only generation.
Accordingly, we define $r_{\text{image}}=r(y\mid v,x)$ and $r_{\text{text}}=r(y\mid v_{\emptyset},x)$.
Comparing $r_{\text{image}}$ and $r_{\text{text}}$ for the same candidate $y$ separates visual evidence from language-only priors; in particular, $r_{\text{text}}>r_{\text{image}}$ indicates that $y$ is more strongly supported under the text-only context, which motivates the rectification in preference construction.

\subsection{Warm-up and On-policy Self-Generation}
We start from a base model $\pi_{\theta_{\text{base}}}$ and perform an SFT warm-up on $\mathcal{D}_{\text{SFT}}$ to obtain a visually grounded instruction-following policy $\pi_{\theta_0}$. which calibrates the initial implicit reward landscape. 

IRIS then proceeds in iterative preference rounds indexed by $r=1, 2, \ldots$ as illustrated in Figure~\ref{fig:fig2}. In round $r$, for each $(v,x)\in\mathcal{D}_{\text{SFT}}$, we sample $K$ candidate responses from the previous-round policy $\pi_{\theta_{r-1}}$:
\begin{equation}
\{y^{(k)}\}_{k=1}^{K} \sim \pi_{\theta_{r-1}}(\cdot \mid v,x).
\end{equation}
These self-generated responses are subsequently scored and sifted to form on-policy preference pairs.

\subsection{Implicit Reward Scoring and Preference Pair Construction}
To construct on-policy preference pairs, we score each sampled response $y$ using length-normalized log-likelihood ratios between the sampling policy $\pi_{\theta_{r-1}}$ and the reference policy $\pi_{\text{ref}}^{(r-1)}$, taken as the preceding policy $\pi_{\theta_{r-2}}$. These ratios are utilized exclusively for scoring and sifting samples into preference pairs.

The image-conditioned implicit reward, which captures grounded alignment under the visual context $v$, is defined as:
\begin{equation}
r_{\text{image}}^{(r)}(v,x,y)
=
\log \frac{\pi_{\theta_{r-1}}(y\mid v,x)}
{\pi_{\text{ref}}^{(r-1)}(y\mid v,x)}.
\label{eq:r_image}
\end{equation}
The text-only implicit reward, which isolates language priors by omitting the visual context via $v_{\emptyset}$, is defined as:
\begin{equation}
r_{\text{text}}^{(r)}(x,y)
=
\log \frac{\pi_{\theta_{r-1}}(y\mid v_{\emptyset},x)}
{\pi_{\text{ref}}^{(r-1)}(y\mid v_{\emptyset},x)}.
\label{eq:r_text}
\end{equation}
To eliminate the influence of response length on the implicit rewards, the log-likelihood is normalized by the token count $|y|$ as $\log \pi(y\mid\cdot) = \frac{1}{|y|} \sum_{t=1}^{|y|} \log \pi(y_t \mid \cdot, y_{<t})$. This normalization prevents the scoring process from being biased by sequence length, ensuring that the selection is solely determined by the grounding quality of the candidates.

We then define a comprehensive grounding-aware score using \textbf{Rectified Visual Guidance (RVG)}:
\begin{equation}
\begin{aligned}
S^{(r)}(v,x,y) &= r_{\text{image}}^{(r)}(v,x,y) \\
&\quad -\gamma \max\!\left(0,\, r_{\text{text}}^{(r)}(x,y)-r_{\text{image}}^{(r)}(v,x,y)\right),
\end{aligned}
\label{eq:rvg}
\end{equation}
where $\gamma \ge 0$ is the rectification strength. The penalty term is activated only when $r_{\text{text}}^{(r)} > r_{\text{image}}^{(r)}$, which corresponds to cases where the model assigns a higher relative likelihood to a response in the absence of visual evidence. This behavior is attributed to an over-reliance on language priors and is identified as a primary source of multimodal hallucinations \citep{xie2024v}. By down-weighting such candidates, RVG enforces a grounding constraint that prioritizes responses derived from actual visual input.

Given the $K$ candidates $\{y^{(k)}\}_{k=1}^{K}$ sampled for each input $(v,x)$, we identify the responses with the maximum and minimum $S^{(r)}$ scores to construct an on-policy preference pair with the highest contrast. This sifting process ensures that the optimization is guided by the most distinct supervisory signals available within the sampled set:
\vskip -0.15in
\begin{equation}
\begin{aligned}
y_w &= \arg\max_{k \in \{1,\ldots,K\}} S^{(r)}(v,x,y^{(k)}), \\
y_l &= \arg\min_{k \in \{1,\ldots,K\}} S^{(r)}(v,x,y^{(k)}).
\end{aligned}
\label{eq:sifting}
\end{equation}
\vskip -0.15in
In round $r$, we designate $y_w$ as the preferred response and $y_l$ as the rejected counterpart for preference optimization.

To ensure high-quality preference data, we apply a lightweight filtering stage with degeneration screening and length-aware filtering to reduce noisy supervision and length bias.
For samples that would otherwise be discarded, we recover them by anchoring the chosen response to the SFT reference from $\mathcal{D}_{\text{SFT}}$.
This improves training stability across rounds while keeping the pipeline self-contained without external evaluators.
We provide the full filtering rules and anchoring details in Appendix~\ref{sec:additional_impl}.

\subsection{Grounded Preference Learning Objectives}

Based on the preference pairs identified through the scoring and sifting process, we update the policy to improve multimodal grounding and alignment. In round $r$, we optimize parameters $\theta$ initialized from $\theta_{r-1}$ and denote the resulting policy as $\pi_{\theta_r}$; for notational simplicity, we write $\pi_\theta$ for the policy being optimized. We minimize a composite objective that combines Conditional Textual Preference, Conditional Visual Preference, and Anchored Regularization. In each round $r$, the reference policy $\pi_{\text{ref}}$ is a frozen copy of the preceding policy $\pi_{\theta_{r-1}}$.

\paragraph{Conditional Textual Preference}
The component $\mathcal{L}_{\text{ctp}}$ adopts the standard DPO objective on $\mathcal{D}^{(r)}$, increasing the preference margin of $y_w$ over $y_l$ under the multimodal context $(v,x)$ relative to the frozen reference policy $\pi_{\text{ref}}$:
\begin{equation}
\begin{aligned}
\mathcal{L}_{\text{ctp}} = -\mathbb{E}_{(v,x,y_w,y_l) \sim \mathcal{D}^{(r)}} \big[ & \log \sigma \big( \beta \log \frac{\pi_{\theta}(y_w \mid v,x)}{\pi_{\text{ref}}(y_w \mid v,x)} \\
& - \beta \log \frac{\pi_{\theta}(y_l \mid v,x)}{\pi_{\text{ref}}(y_l \mid v,x)} \big) \big].
\end{aligned}
\label{eq:l_ctp}
\end{equation}
This term serves as the core preference-learning signal on the sifted on-policy pairs.

\paragraph{Conditional Visual Preference}
The component $\mathcal{L}_{\text{cvp}}$ encourages visual dependence by preferring the same response $y_w$ under the original image $v$ over a perturbed counterpart $\tilde{v}$ where the evidence supporting $y_w$ is suppressed:
\begin{equation}
\begin{aligned}
\mathcal{L}_{\text{cvp}} = -\mathbb{E}_{(v,x,y_w) \sim \mathcal{D}^{(r)}} \big[ & \log \sigma \big( \beta \log \frac{\pi_{\theta}(y_w \mid v,x)}{\pi_{\text{ref}}(y_w \mid v,x)} \\
& - \beta \log \frac{\pi_{\theta}(y_w \mid \tilde{v},x)}{\pi_{\text{ref}}(y_w \mid \tilde{v},x)} \big) \big].
\end{aligned}
\label{eq:l_cvp}
\end{equation}
Here, $\tilde{v}$ is generated by applying a perturbation operator $T(\cdot)$ to $v$.
This term discourages high relative reward for $y_w$ under visually uninformative inputs, thereby promoting grounded preference learning.

\paragraph{Anchored Regularization}
The component $\mathcal{L}_{\text{anchor}}$ stabilizes training by preventing the likelihood of preferred responses from drifting downward. Since DPO-style objectives enforce only a \emph{relative} margin between $(y_w,y_l)$, the preference gap can increase even if the likelihood of $y_w$ decreases. We therefore introduce an anchored term~\citep{wang2024mdpo, yang2025mitigating, liu2025mitigating} that keeps the reference-relative reward of $y_w$ above a soft margin $\delta$:
\begin{equation}
\begin{aligned}
\mathcal{L}_{\text{anchor}} = -\mathbb{E}_{(v,x,y_w) \sim \mathcal{D}^{(r)}} \big[ & \log \sigma \big( \beta \log \frac{\pi_{\theta}(y_w \mid v,x)}{\pi_{\text{ref}}(y_w \mid v,x)} \\
& - \delta \big) \big].
\end{aligned}
\label{eq:l_anchor}
\end{equation}
Here, $\delta$ specifies a soft margin on the reference-relative reward of $y_w$.

\paragraph{Total Objective}
By combining these components, the total objective is defined as the weighted sum of the grounded learning signals:
\begin{equation}
\mathcal{L}_{\text{total}} = \mathcal{L}_{\text{ctp}} + \lambda \mathcal{L}_{\text{cvp}} + \mathcal{L}_{\text{anchor}},
\label{eq:l_total}
\end{equation}

where $\lambda$ controls the strength of $\mathcal{L}_{\text{cvp}}$.

\paragraph{Iterative Alignment with Separate References.}
IRIS starts from a base model $\pi_{\theta_{\text{base}}}$ and performs an SFT warm-up to obtain $\pi_{\theta_0}$. 
For each preference round, we use two references for different purposes. 
When constructing the on-policy preference set, we score self-generated candidates using implicit reward ratios computed between two consecutive policies: in round $r$, scoring uses $\pi_{\theta_{r-1}}$ with $\pi_{\theta_{r-2}}$ as the reference for the log-ratio. For $r=1$, we use the base model $\pi_{\theta_{\text{base}}}$ as the scoring reference. During preference optimization, we initialize the trainable policy from the previous round and use a frozen copy of it as the DPO reference within the round, namely $\pi_\theta \leftarrow \pi_{\theta_{r-1}}$ and $\pi_{\text{ref}} \leftarrow \pi_{\theta_{r-1}}$. 
We then minimize Eq.~\ref{eq:l_total} on $\mathcal{D}^{(r)}$ to obtain the updated policy $\pi_{\theta_r}$, and repeat for a small number of rounds.

\section{Experiments}
\label{sec:exp}

\subsection{Experimental Setup}
\label{sec:exp_setup}

\paragraph{Implementation Details.}
\texttt{IRIS} is implemented on LLaVA-1.5 7B and 13B. The models use CLIP ViT-L/14 and Vicuna-v1.5 as backbones. We set $\gamma = 0.7$ for RVG and use $K = 5$ for on-policy sampling. The generation temperature is 0.7. SFT warm-up uses 5,700 samples from the RLHF-V dataset. For training, we set $\beta = 0.1$ and $\lambda = 1.0$ for $\mathcal{L}_{\text{cvp}}$. The 7B and 13B models use learning rates of $5 \times 10^{-7}$ and $1 \times 10^{-6}$, respectively. Each round is trained for 2 epochs. All experiments are run on 8 NVIDIA H20 GPUs. The entire pipeline operates without any external human or AI feedback. Details on constructing the rejected images $\tilde{v}$ are provided in Appendix~\ref{sec:additional_impl}.

\paragraph{Evaluation Benchmarks.}
We evaluate \texttt{IRIS} on several representative benchmarks to assess both hallucination mitigation and general capabilities. \textbf{AMBER}~\citep{wang2023amber} is a multi-dimensional generative benchmark with 1,004 images; we report object hallucination (\textbf{CHAIR}$\downarrow$), coverage (\textbf{Cover}$\uparrow$), response-level hallucination (\textbf{HalRate}$\downarrow$), and cognitive hallucination (\textbf{Cog}$\downarrow$). \textbf{MMHal-Bench}~\citep{sun2024aligning} includes 96 images across 12 categories for question answering; we follow the official rubric to report the overall quality \textbf{Score}$\uparrow$ and \textbf{HalRate}$\downarrow$ using GPT-4 evaluation. \textbf{Object-HalBench}~\citep{rohrbach2018object} consists of 300 instances for image description; we report CHAIR metrics at both the sentence (\textbf{CHAIRs}$\downarrow$) and instance levels (\textbf{CHAIRi}$\downarrow$). Finally, \textbf{LLaVA-Bench (in-the-Wild)}~\citep{liu2023visual} is used to assess general conversational ability via GPT-4-relative scores.

\vskip -0.1in
\paragraph{Baselines.}
We compare \texttt{IRIS} against a comprehensive set of recent state-of-the-art approaches for hallucination mitigation, including \textbf{LLaVA-RLHF}~\citep{sun-etal-2024-aligning}, \textbf{HALVA}~\citep{sarkar2024mitigating}, \textbf{mDPO}~\citep{wang2024mdpo}, \textbf{HA-DPO}~\citep{zhao2023beyond}, \textbf{V-DPO}~\citep{xie2024v}, \textbf{POVID}~\citep{zhou2024aligning}, \textbf{RLAIF-V}~\citep{yu2024rlaif}, \textbf{SymMPO}~\citep{liu2025mitigating}, \textbf{RLHF-V (HD)}~\citep{yu2024rlhf}, \textbf{LPOI}~\citep{zadeh2025lpoi}, and \textbf{OPA-DPO}~\citep{yang2025mitigating}.

\subsection{Main Results}
\label{sec:main_results}

Table~\ref{tab:comparison_SOTA} summarizes the main results on three representative hallucination benchmarks. We report \texttt{IRIS}-R2 as our final model. Overall, \texttt{IRIS}-R2 improves grounding-oriented performance on both 7B and 13B backbones. On AMBER, it reduces hallucination-related metrics such as CHAIR and HalRate compared to the vanilla \texttt{LLaVA-1.5} models, while keeping coverage at a similar level. On MMHal-Bench, \texttt{IRIS}-R2 also improves over the base models, but the gains are smaller than those on object-level hallucination metrics; this may be partly because MMHal-Bench emphasizes compositional visual reasoning (e.g., counting and relations), while our method focuses on improving visual grounding to reduce hallucinated objects and attributes. Notably, on Object HalBench, \texttt{IRIS}-R2 achieves \textbf{8.66} CHAIRs, showing strong improvements in fine-grained object grounding. Furthermore, we observe that \texttt{IRIS}-R2 consistently outperforms \texttt{IRIS}-R1 across model scales, validating the effectiveness of iterative self-alignment.

We further compare \texttt{IRIS} with two recent strong baselines, RLAIF-V~\citep{yu2024rlaif} and OPA-DPO~\citep{yang2025mitigating}.
Compared to RLAIF-V, \texttt{IRIS} is competitive on AMBER and achieves strong object-level grounding on Object HalBench.
Compared to OPA-DPO, which relies on GPT-4V feedback, \texttt{IRIS} remains competitive on AMBER while achieving clear gains on Object HalBench and attaining higher coverage.
Crucially, regarding efficiency, while RLAIF-V also employs open-source models, it relies on heavy prompt-based labeler scoring. In contrast, \texttt{IRIS} leverages implicit rewards, resulting in an approximately $40\times$ reduction in curation cost (Appendix~\ref{app:computational_cost}).

\definecolor{rowgray}{gray}{0.95}
\definecolor{oursbg}{RGB}{255, 240, 230}
\definecolor{myteal}{RGB}{0, 128, 128}
\definecolor{myred}{RGB}{178, 34, 34}
\providecommand{\imp}[1]{\textcolor{myteal}{\tiny{\textbf{(#1)}}}}
\providecommand{\dec}[1]{\textcolor{myred}{\tiny{\textbf{(#1)}}}}

\begin{table*}[t]
\caption{
Comparative assessment of \texttt{IRIS} against state-of-the-art baselines on multimodal hallucination benchmarks.
Boldface indicates the best result.
Values in parentheses denote the relative change with respect to the corresponding vanilla \texttt{LLaVA-1.5} backbone.
}
\label{tab:comparison_SOTA}
\vskip 0.15in
\begin{center}
\begin{small}
\begin{sc}
\renewcommand{\arraystretch}{1.2}
\setlength{\tabcolsep}{2.5pt}
\resizebox{\textwidth}{!}{
\begin{tabular}{l c c | c c c c | c c | c c }
\toprule
\multirow{3}{*}{\textbf{Algorithm}} & \multirow{3}{*}{\textbf{Data Size}} & \multirow{3}{*}{\textbf{Feedback}} & \multicolumn{4}{c|}{\textbf{AMBER}} & \multicolumn{2}{c|}{\textbf{MMHal}} & \multicolumn{2}{c}{\textbf{Object Hal}} \\
& & & \textbf{CHAIR}$\downarrow$ & \textbf{Cover}$\uparrow$ & \textbf{HalRate}$\downarrow$ & \textbf{Cog}$\downarrow$ & \textbf{Score}$\uparrow$ & \textbf{HalRate}$\downarrow$ & \textbf{CHAIRs}$\downarrow$ & \textbf{CHAIRi}$\downarrow$ \\
\midrule
\rowcolor{rowgray} GPT-4V ~\citep{yang2023dawn} & \(\times\) &  \(\times\)& 4.6 & 67.1 & 30.7 & 2.6 & 3.49 & 0.28 & 13.6 & 7.3 \\
\rowcolor{rowgray} Qwen-VL-Chat ~\citep{Qwen-VL} &  \(\times\) &  \(\times\)& 6.6 & 53.2 & 33.2 & 31.0 & 2.89 & 0.60 & 36.0 & 21.3 \\
\rowcolor{rowgray} Silkic ~\citep{li2025mini} &  \(\times\)&  \(\times\)& 5.4 & 55.8 & 29.0 & 2.0 & 3.01 & 0.41 & 25.3 & 13.9 \\
\rowcolor{rowgray} InstructBLIP ~\citep{dai2023instructblip} &  \(\times\)&  \(\times\)& 5.4 & 55.2 & 29.0 & 38.2 & 2.21 & 1.35 & 25.9 & 14.3 \\
\rowcolor{rowgray} MiniGemini ~\citep{li2025mini} &  \(\times\)&  \(\times\)& 5.8 & 55.8 & 29.0 & 3.08 & 0.38 & 0.38 & 14.5 & 8.0 \\

\midrule
\rowcolor{rowgray} \textbf{LLaVA-1.5-7B ~\citep{liu2023visual}} & & & 8.8 & 50.1 & 40.4 & 4.7 & 2.18 & 0.59 & 54.70 & 26.5 \\
+LLaVA-RLHF ~\citep{sun-etal-2024-aligning} & 122k & RLHF & 9.7 & \textbf{53.2} & 46.6 & 5.3 & 1.88 & 0.71 & 58.00 & 15.61 \\
+HALVA ~\citep{sarkar2024mitigating} & 21.5k & GPT-4V & 6.6 & 53.0 & 32.2 & 3.4 & 2.25 & 0.54 & 41.40 & 11.70 \\
+mDPO ~\citep{wang2024mdpo} & 10k & GPT-4V & 4.4 & 52.4 & 24.5 & 2.4 & 2.39 & 0.54 & 35.70 & 9.80 \\

+HA-DPO ~\citep{zhao2023beyond} & 6k & GPT-4 & 7.8 & 52.1 & 35.6 & 4.2 & 1.89 & 0.65 & 54.00 & 14.45 \\
+V-DPO ~\citep{xie2024v} & 10K & GPT-3.5 & 6.6 & 49.1 & 30.8 & 3.1 & - & - & - & - \\
+POVID ~\citep{zhou2024aligning} & 17k & GPT-4V & 7.4 & 51.3 & 34.3 & 3.9 & 2.08 & 0.60 & 50.67 & 15.28 \\
+RLAIF-V ~\citep{yu2024rlaif} & 16k & LLaVA-Next & 3.0 & 50.4 & 16.2 & 1.0 & \textbf{3.00} & \textbf{0.38} & 16.00 & \textbf{3.70} \\
+SymMPO ~\cite{liu2025mitigating} & 21K & Deepseek-V3 & 5.2 & 49.5 & 25.4 & 3.0 & 2.63 & 0.51 & 20.4 & 10.3 \\
+OPA-DPO ~\citep{yang2025mitigating} & 4.8k & GPT-4V &\textbf{2.2} &47.9 &11.6 &\textbf{0.9} &2.83 &0.45 &13.00 &4.25 \\
+LPOI ~\cite{zadeh2025lpoi} & 10K & GPT-4V & 4.3 & 51.9 & 26.4 & 2.0 & 2.40 & 0.59 & 24.3 & 14.6 \\

\rowcolor{oursbg} \textbf{+IRIS-R1 (Ours)} & 5.7k & implicit reward & 3.8\imp{-5.0} & 51.9\imp{+1.8} & 17.5\imp{-22.9} & 1.6\imp{-3.1} & 2.34\imp{+0.16} & 0.50\imp{-0.09} & 17.3\imp{-37.4} & 8.45\imp{-18.05} \\
\rowcolor{oursbg} \textbf{+IRIS-R2 (Ours)} & 5.7k & implicit reward & 2.4\imp{-6.4} & 51.1\imp{+1.0} & \textbf{11.3}\imp{-29.1} & 1.1\imp{-3.6} & 2.42\imp{+0.24} & 0.50\imp{-0.09} & \textbf{8.66}\imp{-46.04} & 4.56\imp{-21.94} \\
\midrule
\rowcolor{rowgray} \textbf{LLaVA-1.5-13B ~\citep{liu2023visual}} & & & 8.8 & 50.3 & 37.2 & 4.3 & 2.31 & 0.55 & 49.3 & 23.9 \\
+LLaVA-RLHF ~\citep{sun-etal-2024-aligning} & 122k & RLHF & 7.7 & 52.3 & 38.6 & 4.0 & 2.27 & 0.64 & 44.67 & 11.83 \\

+mDPO ~\citep{wang2024mdpo} & 10K & GPT-4V & 4.6 & 52.6 & 25.0 & 2.0 & 2.50 & 0.57 & 33.3 & 16.6 \\
+RLHF-V (HD) ~\citep{yu2024rlhf} & 1.4k & Human & 6.3 & 46.1 & 25.1 & 2.1 & 2.81 & 0.49 & - & - \\

+HALVA  ~\citep{sarkar2024mitigating} & 21.5k & GPT-4V & 6.4 & 52.6 & 30.4 & 3.2 & 2.58 & 0.45 & 45.40 & 12.80 \\
+SymMPO ~\citep{liu2025mitigating} & 21k & Deepseek-V3 & 4.8 & 52.8 & 25.1 & 2.1 & 2.85 & 0.48 & 18.3 & 10.0 \\
+LPOI ~\cite{zadeh2025lpoi}& 10K & GPT-4V & 3.9 & 52.9 & 22.3 & 1.8 & 2.54 & 0.57 & 24.3 & 11.7 \\
\rowcolor{oursbg} \textbf{+IRIS-R1 (Ours)} & 5.7k & implicit reward & 3.7\imp{-5.1} & \textbf{53.7}\imp{+3.4} & 20.2\imp{-17.0} & 1.9\imp{-2.4} & 2.82\imp{+0.51} & 0.42\imp{-0.14} & 18.6\imp{-30.7} & 9.1\imp{-14.8} \\
\rowcolor{oursbg} \textbf{+IRIS-R2 (Ours)} & 5.7k & implicit reward & \textbf{3.5}\imp{-5.3} & 52.2\imp{+1.9} & \textbf{18}\imp{-19.2} & \textbf{1.7}\imp{-2.6} & \textbf{2.86}\imp{+0.55} & \textbf{0.41}\imp{-0.14} & \textbf{10}\imp{-39.3} & \textbf{5.49}\imp{-18.41} \\
\bottomrule
\end{tabular}}
\end{sc}
\end{small}
\end{center}
\vskip -0.1in
\end{table*}

\subsection{Ablation Studies}

\paragraph{Effect of Objective Components.}
Table~\ref{tab:loss_ablation} isolates the effect of each objective component. The results show that conditional visual preference, denoted by $\mathcal{L}_{\text{cvp}}$, yields the largest improvement on CHAIR-based hallucination metrics. Conditional visual preference is the main signal for grounding, while anchored regularization, $\mathcal{L}_{\text{anchor}}$, helps stabilize training and prevent capability degradation. Removing anchored regularization leads to a drop in general capability below the vanilla base model, as further evidenced in Appendix~\ref{tab:general_capability_ablation}.

\begin{table}[htp!]
    \caption{Ablation study on objective components. Starting from $\mathcal{L}_{\text{ctp}}$, we add $\mathcal{L}_{\text{cvp}}$ and $\mathcal{L}_{\text{anchor}}$. The full objective achieves the best overall results, while $\mathcal{L}_{\text{cvp}}$ provides the primary gains.}
    \label{tab:loss_ablation}
    \vskip -0.10in
    \begin{center}
    \begin{small}
    \begin{sc}
    \setlength{\tabcolsep}{3.6pt}
    \renewcommand{\arraystretch}{1.15}
    \resizebox{\linewidth}{!}{
    \begin{tabular}{c c c | c c c | c c}
        \toprule
        \multicolumn{3}{c|}{Components} & \multicolumn{3}{c|}{AMBER} & \multicolumn{2}{c}{Object Hal} \\
        \cmidrule(lr){1-3}\cmidrule(lr){4-6}\cmidrule(lr){7-8}
        $\mathcal{L}_{\text{ctp}}$ & $\mathcal{L}_{\text{cvp}}$ & $\mathcal{L}_{\text{anchor}}$ 
        & CHAIR$\downarrow$ & HalRate$\downarrow$ & Cog$\downarrow$ & CHAIRs$\downarrow$ & CHAIRi$\downarrow$ \\
        \midrule
        $\checkmark$ & $\times$      & $\times$      & 5.8 & 30.7 & 2.0 & 18.0 & 7.98 \\
        $\checkmark$ & $\times$      & $\checkmark$  & 4.9 & 25.4 & 2.2 & 19.3 & 9.61 \\
        $\checkmark$ & $\checkmark$  & $\times$      & 2.9 & \textbf{10.4} & \textbf{0.8} & 10.2 & 4.87 \\
        \rowcolor{rowgray}
        $\checkmark$ & $\checkmark$  & $\checkmark$  & \textbf{2.4} & 11.3 & 1.1 &8.66 & \textbf{4.56} \\
        \bottomrule
    \end{tabular}
    }
    \end{sc}
    \end{small}
    \end{center}
    \vskip -0.3in
\end{table}

\paragraph{Impact of Training Paradigms.}
Table~\ref{tab:paradigm_ablation_factorized} studies two factors in our training pipeline: \textbf{SFT warm-up} and on-policy \textbf{self-generation}. We first find that the SFT warm-up gives a clearly better starting point: with the same training round, models with SFT warm-up consistently show lower hallucination metrics than those trained without it. This suggests that SFT helps the policy learn a more grounded response pattern before preference optimization.

More importantly, on-policy self-generation brings the largest improvements across rounds. When we replace static human pairs with self-generated pairs, the model reduces hallucinations faster and reaches a much better final result, especially on Object HalBench. In contrast, training with static human pairs yields more limited gains, which diminish as training proceeds. Overall, the results suggest that SFT warm-up improves the initial model state, whereas on-policy self-generation enables sustained improvement across successive rounds.

\begin{table}[!t]
    \caption{Factorized ablation on training paradigms. We evaluate the impact of \textbf{SFT warm-up} and on-policy \textbf{Self-gen} compared to off-policy human data (RLHF-V).}
  \label{tab:paradigm_ablation_factorized}
  \vskip -0.05in
  \centering
  \scriptsize
  \setlength{\tabcolsep}{3.8pt}
  \renewcommand{\arraystretch}{1.12}
  \resizebox{\linewidth}{!}{
  \begin{tabular}{c c c | c c c | c c}
    \toprule
    \multicolumn{2}{c}{\textbf{Factors}} & \multirow{2}{*}{\textbf{Round}} &
    \multicolumn{3}{c|}{\textbf{AMBER}} & \multicolumn{2}{c}{\textbf{Object Hal}} \\
    \cmidrule(lr){1-2}\cmidrule(lr){4-6}\cmidrule(lr){7-8}
    \textbf{SFT warm-up} & \textbf{Self-gen} & &
    CHAIR$\downarrow$ & HalRate$\downarrow$ & Cog$\downarrow$ &
    CHAIRs$\downarrow$ & CHAIRi$\downarrow$ \\
    \midrule

    $\times$ & $\times$ & R0 & 5.9 & 29.8 & 3.3 & 43.3 & 21.3 \\
    $\times$ & $\times$ & R1 & 4.9 & 25.3 & 2.8 & 31.0 & 15.3 \\
    \rowcolor{rowgray} $\times$ & $\times$ & R2 & 3.8 & 20.2 & 2.1 & 27.3 & 13.0 \\
    \midrule

    $\checkmark$ & $\times$ & R0 & 5.3 & 25.5 & 2.5 & 24.0 & 13.0 \\
    $\checkmark$ & $\times$ & R1 & 4.6 & 23.1 & 1.9 & 19.0 & 9.5 \\
    \rowcolor{rowgray} $\checkmark$ & $\times$ & R2 & 3.7 & 19.0 & 1.7 & 17.3 & 8.9 \\
    \midrule

    $\times$ & $\checkmark$ & R0 & 5.9 & 29.8 & 3.3 & 43.3 & 21.3 \\
    $\times$ & $\checkmark$ & R1 & 3.6 & 19.8 & 2.0 & 23.6 & 11.0 \\
    \rowcolor{rowgray} $\times$ & $\checkmark$ & R2 & 2.5 & 14.4 & 1.5 & 18.3 & 9.22 \\
    \midrule

    $\checkmark$ & $\checkmark$ & R0 & 5.3 & 25.5 & 2.5 & 24.0 & 13.0 \\
    $\checkmark$ & $\checkmark$ & R1 & 3.8 & 17.5 & 1.6 & 17.3 & 8.45 \\
    \rowcolor{rowgray} $\checkmark$ & $\checkmark$ & R2 & \textbf{2.4} & \textbf{11.3} & \textbf{1.1} & \textbf{8.6} & \textbf{4.56} \\
    \bottomrule
  \end{tabular}}
  \vskip -0.2in
\end{table}

\paragraph{Effectiveness of Scoring Signals.}
Table~\ref{tab:reward_ablation} compares alternative scoring signals for sifting preference pairs. In Round 1, the unimodal scores $r_{\text{text}}$ and $r_{\text{image}}$ yield comparable results and neither consistently dominates across metrics. In Round 2, $r_{\text{image}}$ improves upon $r_{\text{text}}$ on most hallucination measures, suggesting that visual conditioning becomes increasingly helpful as the policy is refined.

RVG performs best across all reported metrics. On Object HalBench, RVG reduces CHAIRs to 8.6 in Round 2, compared to 14.1 when sifting with $r_{\text{image}}$. This pattern is consistent with RVG suppressing candidates that remain highly preferred under text-only priors but are weakly supported by the image, thereby producing more informative preference pairs for subsequent optimization.

\begin{table}[!ht]
    \caption{Ablation on scoring signals for sifting. While using the image-only reward ($r_{\text{image}}$) is more effective than the text-only signal ($r_{\text{text}}$), our proposed RVG achieves the best performance.}
    \label{tab:reward_ablation}
    \vskip -0.4in
    \begin{center}
    \begin{small}
    \begin{sc}
    \setlength{\tabcolsep}{3.2pt}
    \renewcommand{\arraystretch}{1.10}
    \resizebox{\linewidth}{!}{
    \begin{tabular}{l c | c c c | c c}
        \toprule
        Scoring & Round & \multicolumn{3}{c|}{AMBER} & \multicolumn{2}{c}{Object Hal} \\
         &  & CHAIR$\downarrow$ & HalRate$\downarrow$ & Cog$\downarrow$ & CHAIRs$\downarrow$ & CHAIRi$\downarrow$ \\
        \midrule
        $r_{\text{text}}$ (Text-only)  & R1 & 3.6 & 17.9 & 1.8 & 17.9 & 9.22 \\
        \rowcolor{rowgray}
        $r_{\text{text}}$ (Text-only)  & R2 & 3.2 & 15.6 & 1.5 & 15.2 & 7.85 \\
        \midrule
        $r_{\text{image}}$ (Image-only) & R1 & 3.8 & 17.8 & 1.8 & 18.3 & 9.30 \\
        \rowcolor{rowgray}
        $r_{\text{image}}$ (Image-only) & R2 & 3.0 & 14.5 & 1.4 & 14.1 & 7.12 \\
        \midrule
        RVG (Ours) & R1 & 3.8 & 17.5 & 1.6 & 17.3 & 8.45 \\
        \rowcolor{rowgray}
        RVG (Ours) & R2 & \textbf{2.4} & \textbf{11.3} & \textbf{1.1} & \textbf{8.6} & \textbf{4.56} \\
        \bottomrule
    \end{tabular}
    }
    \end{sc}
    \end{small}
    \end{center}
    \vskip -0.2in
\end{table}

\section{Hyperparameter Sensitivity}
\label{sec:hparam_sensitivity}

We examine the sensitivity of two key hyperparameters in \texttt{IRIS}: the rectification strength $\gamma$ in RVG, which controls the penalty for unsupported language priors, and the weight $\lambda$, which balances the conditional visual preference term.
Figures~\ref{fig:hparam_sensitivity} and \ref{fig:hparam_lambda} summarize the trends across a wide range of values.
Overall, the performance is robust, with the optimal results achieved around $\gamma=0.7$ and $\lambda=1.0$. We adopt these as default settings in all subsequent experiments.
Detailed numerical results and comprehensive sensitivity analyses are presented in Appendix~\ref{app:hparam_sensitivity}.

\begin{figure}[!htp]
    \centering
    \includegraphics[width=\linewidth]{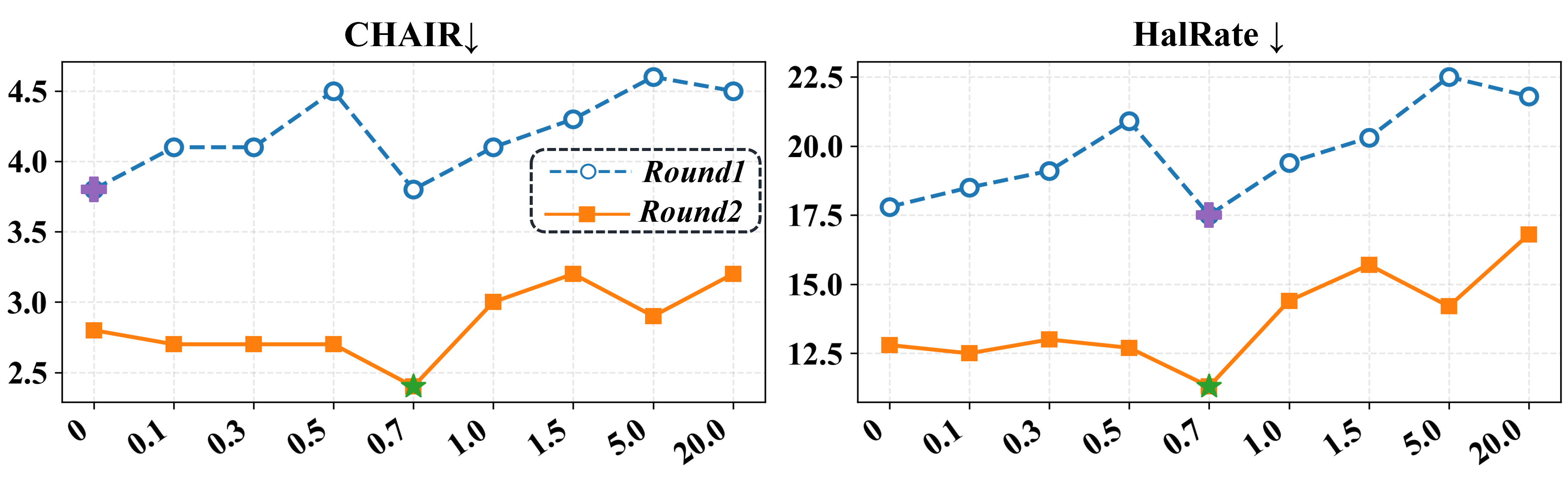}
    \caption{\textbf{Effect of Rectification Strength $\gamma$.} Sensitivity of hallucination metrics to $\gamma$ across two iterative rounds. The star indicates the optimal value at $\gamma=0.7$.}
    \vskip -0.1in
    \label{fig:hparam_sensitivity}
\end{figure}

\begin{figure}[!htp]
    \centering
    \includegraphics[width=\linewidth]{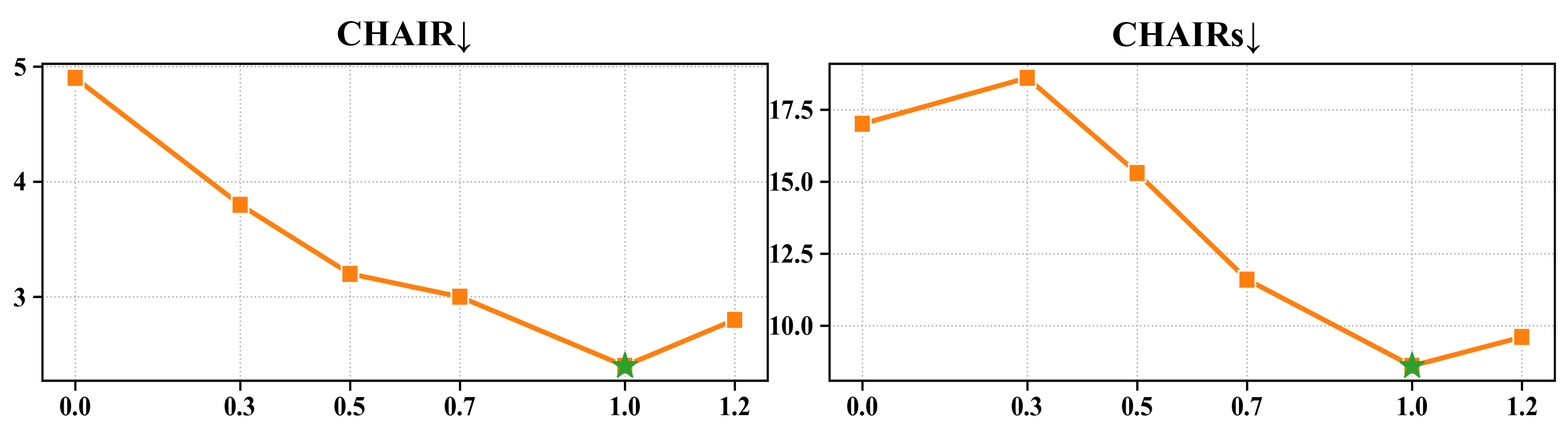}
    \caption{\textbf{Effect of Visual Preference Weight $\lambda$.} Sensitivity of hallucination metrics to the weight $\lambda$ in the final round. The star indicates the optimal value at $\lambda=1.0$.} 
    \vskip -0.2in
    \label{fig:hparam_lambda}
\end{figure}

\section{Further Analysis}
\label{sec:further_analysis}

\paragraph{Data Efficiency.}
Table~\ref{tab:data_efficiency} examines how \texttt{IRIS} changes with the amount of on-policy preference data. With only 1k pairs, \texttt{IRIS} already improves hallucination metrics on both AMBER and Object HalBench, compared to the model at the start of preference training. Increasing the budget to 3k pairs leads to much larger gains, and using 5.7k pairs gives the best or tied-best results on most metrics. Across all three budgets, Round 2 consistently outperforms Round 1, showing that iterative on-policy refinement remains effective without requiring tens of thousands of preference pairs.

\begin{table}[!ht]
    \caption{Data efficiency analysis. IRIS achieves strong performance with limited training data, demonstrating sample efficiency.}
    \label{tab:data_efficiency}
    \vskip 0.15in
    \begin{center}
    \begin{small}
    \begin{sc}
    \resizebox{\linewidth}{!}{
    \begin{tabular}{l c | c c c | c c}
        \toprule
        \multirow{2}{*}{Data} & \multirow{2}{*}{Round} & \multicolumn{3}{c|}{AMBER} & \multicolumn{2}{c}{Object Hal} \\
         & & CHAIR$\downarrow$ & HalRate$\downarrow$ & Cog$\downarrow$ & CHAIRs$\downarrow$ & CHAIRi$\downarrow$ \\
        \midrule
        1k & R1 & 5.2 & 25.5 & 2.6 & 25.0 & 13.0 \\
        \rowcolor{rowgray}1k & R2 & 4.9 & 23.8 & 2.2 & 24.3 & 12.3 \\
        \midrule
        3k & R1 & 4.8 & 22.6 & 2.1 & 20.6 & 10.7 \\
        \rowcolor{rowgray}3k & R2 & 2.9 & 13.3 & \textbf{1.1} & 11.0 & 5.54 \\
        \midrule
        5.7k & R1 & 3.8 & 17.5 & 1.6 & 17.3 & 8.45 \\
        \rowcolor{rowgray}5.7k & R2 & \textbf{2.4} & \textbf{11.3} & \textbf{1.1} & \textbf{8.6} & \textbf{4.56} \\
        \bottomrule
    \end{tabular}
    }
    \end{sc}
    \end{small}
    \end{center}
    \vskip -0.2in
\end{table}

\paragraph{Robustness to Sampling Repeat Times $K$.}
Table~\ref{tab:repeat} studies how sensitive \texttt{IRIS} is to the sampling repeat factor $K$. Performance stays similar across $K \in \{3,5,10\}$, and each setting improves from Round 1 to Round 2. The final results differ only a little, which suggests that \texttt{IRIS} does not need a large sampling budget to work well. Using more candidates can help, but the gains become smaller as $K$ increases. We set $K=5$ as a simple default that gives strong results at a reasonable cost, and the method remains robust under other choices of $K$.

\begin{table}[!ht]
    \caption{Ablation on sampling repeat times $K$ for on-policy data generation.}
    \label{tab:repeat}
    \vskip -0.7in
    \begin{center}
    \begin{small}
    \begin{sc}
    \resizebox{\linewidth}{!}{
    \begin{tabular}{l c | c c c | c c}
        \toprule
        \multirow{2}{*}{Repeat $K$} & \multirow{2}{*}{Round} & \multicolumn{3}{c|}{AMBER} & \multicolumn{2}{c}{Object Hal} \\
         & & CHAIR$\downarrow$ & HalRate$\downarrow$ & Cog$\downarrow$ & CHAIRs$\downarrow$ & CHAIRi$\downarrow$ \\
        \midrule
        $K=3$ & R1 & 3.6 & 16.7 & 1.7 & 17.6 & 9.28 \\
        \rowcolor{rowgray}$K=3$ & R2 & 2.8 & 12.1 & 1.2 & 7.66 & 4.10 \\
        \midrule
        $K=5$  & R1 & 3.8 & 17.5 & 1.6 & 17.3 & 8.45 \\
        \rowcolor{rowgray}$K=5$ & R2 & \textbf{2.4} & \textbf{11.3} & \textbf{1.1} & 8.6 & 4.56 \\
        \midrule
        $K=10$ & R1 & 3.6 & 16.3 & 1.5 & 16.3 & 7.75 \\
        \rowcolor{rowgray}$K=10$ & R2 & 2.7 & 12.2 & 1.5 & \textbf{6.0} & \textbf{3.55} \\
        
        \bottomrule
        
    \end{tabular}
    }
    \end{sc}
    \end{small}
    \end{center}
    \vskip -0.2in
\end{table}

\section{Conclusion}
\label{sec:conclusion}

We presented \texttt{IRIS}, an iterative on-policy self-alignment framework for mitigating hallucinations in MLLMs. The core of our approach is demonstrating that intrinsic implicit rewards can be effectively harnessed to identify high-quality preference signals from a model's own generative distribution, thereby eliminating the dependency on costly external evaluators or proprietary models. By incorporating RVG during the sifting process, \texttt{IRIS} successfully isolates visual evidence from language priors, enabling the model to refine its grounding through iterative refinement cycles.

Experimental results confirm that this paradigm consistently improves object-level grounding across multiple benchmarks. Our analysis further shows that \texttt{IRIS} is both sample-efficient and robust to hyperparameter choices, narrowing the performance gap to methods that rely on high-cost external feedback. Overall, by providing a practical and principled approach for internal preference mining, we believe \texttt{IRIS} offers a new and efficient perspective for mitigating hallucinations in future multimodal models.

\newpage

\nocite{langley00}

\bibliography{example_paper}

\begin{thebibliography}{44}
\providecommand{\natexlab}[1]{#1}
\providecommand{\url}[1]{\texttt{#1}}
\expandafter\ifx\csname urlstyle\endcsname\relax
  \providecommand{\doi}[1]{doi: #1}\else
  \providecommand{\doi}{doi: \begingroup \urlstyle{rm}\Url}\fi

\bibitem[Achiam et~al.(2023)Achiam, Adler, Agarwal, Ahmad, Akkaya, Aleman, Almeida, Altenschmidt, Altman, Anadkat, et~al.]{achiam2023gpt}
Achiam, J., Adler, S., Agarwal, S., Ahmad, L., Akkaya, I., Aleman, F.~L., Almeida, D., Altenschmidt, J., Altman, S., Anadkat, S., et~al.
\newblock Gpt-4 technical report.
\newblock \emph{arXiv preprint arXiv:2303.08774}, 2023.

\bibitem[Bai et~al.(2023)Bai, Bai, Yang, Wang, Tan, Wang, Lin, Zhou, and Zhou]{Qwen-VL}
Bai, J., Bai, S., Yang, S., Wang, S., Tan, S., Wang, P., Lin, J., Zhou, C., and Zhou, J.
\newblock Qwen-vl: A frontier large vision-language model with versatile abilities.
\newblock \emph{arXiv preprint arXiv:2308.12966}, 2023.

\bibitem[Chen et~al.(2024{\natexlab{a}})Chen, Liu, Du, Pang, Liu, Sinha, Varakantham, and Lin]{chen2024bootstrapping}
Chen, C., Liu, Z., Du, C., Pang, T., Liu, Q., Sinha, A., Varakantham, P., and Lin, M.
\newblock Bootstrapping language models with dpo implicit rewards.
\newblock \emph{arXiv preprint arXiv:2406.09760}, 2024{\natexlab{a}}.

\bibitem[Chen et~al.(2024{\natexlab{b}})Chen, Zhao, Luo, Yao, Li, and Zhou]{chen2024halc}
Chen, Z., Zhao, Z., Luo, H., Yao, H., Li, B., and Zhou, J.
\newblock Halc: Object hallucination reduction via adaptive focal-contrast decoding.
\newblock \emph{arXiv preprint arXiv:2403.00425}, 2024{\natexlab{b}}.

\bibitem[Dai et~al.(2023)Dai, Li, Li, Tiong, Zhao, Wang, Li, Fung, and Hoi]{dai2023instructblip}
Dai, W., Li, J., Li, D., Tiong, A., Zhao, J., Wang, W., Li, B., Fung, P.~N., and Hoi, S.
\newblock Instructblip: Towards general-purpose vision-language models with instruction tuning.
\newblock \emph{Advances in neural information processing systems}, 36:\penalty0 49250--49267, 2023.

\bibitem[Fu et~al.(2025)Fu, Fei, Shen, Hooi, Qiu, Ng, et~al.]{fu2025chip}
Fu, J., Fei, H., Shen, X., Hooi, B., Qiu, X., Ng, S.-K., et~al.
\newblock Chip: Cross-modal hierarchical direct preference optimization for multimodal llms.
\newblock In \emph{The Thirteenth International Conference on Learning Representations}, 2025.

\bibitem[Geng et~al.(2025)Geng, Ivison, Li, Sap, Li, Krishna, and Koh]{geng2025delta}
Geng, S., Ivison, H., Li, C.-L., Sap, M., Li, J., Krishna, R., and Koh, P.~W.
\newblock The delta learning hypothesis: Preference tuning on weak data can yield strong gains.
\newblock \emph{arXiv preprint arXiv:2507.06187}, 2025.

\bibitem[Guo et~al.(2024)Guo, Zhang, Liu, Liu, Khalman, Llinares, Rame, Mesnard, Zhao, Piot, et~al.]{guo2024direct}
Guo, S., Zhang, B., Liu, T., Liu, T., Khalman, M., Llinares, F., Rame, A., Mesnard, T., Zhao, Y., Piot, B., et~al.
\newblock Direct language model alignment from online ai feedback.
\newblock \emph{arXiv preprint arXiv:2402.04792}, 2024.

\bibitem[He et~al.()He, Chen, Shi, Yu, Shao, and Sheng]{hesystematic}
He, L., Chen, Z., Shi, Z., Yu, T., Shao, J., and Sheng, L.
\newblock Systematic reward gap optimization for mitigating vlm hallucinations.
\newblock In \emph{The Thirty-ninth Annual Conference on Neural Information Processing Systems}.

\bibitem[Ko et~al.(2025)Ko, Dingliwal, Ganesh, Sengupta, Bodapati, and Galstyan]{ko2025sera}
Ko, J., Dingliwal, S., Ganesh, B., Sengupta, S., Bodapati, S.~B., and Galstyan, A.
\newblock Sera: Self-reviewing and alignment of llms using implicit reward margins.
\newblock In \emph{The Thirteenth International Conference on Learning Representations}, 2025.

\bibitem[Kullback \& Leibler(1951)Kullback and Leibler]{kullback1951information}
Kullback, S. and Leibler, R.~A.
\newblock On information and sufficiency.
\newblock \emph{The annals of mathematical statistics}, 22\penalty0 (1):\penalty0 79--86, 1951.

\bibitem[Langley(2000)]{langley00}
Langley, P.
\newblock Crafting papers on machine learning.
\newblock In Langley, P. (ed.), \emph{Proceedings of the 17th International Conference on Machine Learning (ICML 2000)}, pp.\  1207--1216, Stanford, CA, 2000. Morgan Kaufmann.

\bibitem[Leng et~al.(2024{\natexlab{a}})Leng, Zhang, Chen, Li, Lu, Miao, and Bing]{leng2024mitigating}
Leng, S., Zhang, H., Chen, G., Li, X., Lu, S., Miao, C., and Bing, L.
\newblock Mitigating object hallucinations in large vision-language models through visual contrastive decoding.
\newblock In \emph{Proceedings of the IEEE/CVF Conference on Computer Vision and Pattern Recognition}, pp.\  13872--13882, 2024{\natexlab{a}}.

\bibitem[Leng et~al.(2024{\natexlab{b}})Leng, Zhang, Chen, Li, Lu, Miao, and Bing]{leng2024mitigatingVCD}
Leng, S., Zhang, H., Chen, G., Li, X., Lu, S., Miao, C., and Bing, L.
\newblock Mitigating object hallucinations in large vision-language models through visual contrastive decoding.
\newblock In \emph{Proceedings of the IEEE/CVF Conference on Computer Vision and Pattern Recognition}, pp.\  13872--13882, 2024{\natexlab{b}}.

\bibitem[Leng et~al.(2025)Leng, Xing, Cheng, Zhou, Zhang, Li, Zhao, Lu, Miao, and Bing]{leng2025the}
Leng, S., Xing, Y., Cheng, Z., Zhou, Y., Zhang, H., Li, X., Zhao, D., Lu, S., Miao, C., and Bing, L.
\newblock The curse of multi-modalities: Evaluating hallucinations of large multimodal models across language, visual, and audio.
\newblock In \emph{The Thirty-ninth Annual Conference on Neural Information Processing Systems Datasets and Benchmarks Track}, 2025.
\newblock URL \url{https://openreview.net/forum?id=G4AZhSEcrV}.

\bibitem[Li et~al.(2025{\natexlab{a}})Li, Yan, Tang, Li, Zheng, and Jin]{li2025mitigatingvisualhallucinationssemantic}
Li, Y., Yan, Y., Tang, J., Li, Y., Zheng, Z., and Jin, Y.
\newblock Mitigating visual hallucinations via semantic curriculum preference optimization in mllms, 2025{\natexlab{a}}.
\newblock URL \url{https://arxiv.org/abs/2509.24491}.

\bibitem[Li et~al.(2025{\natexlab{b}})Li, Zhang, Wang, Zhong, Chen, Chu, Liu, and Jia]{li2025mini}
Li, Y., Zhang, Y., Wang, C., Zhong, Z., Chen, Y., Chu, R., Liu, S., and Jia, J.
\newblock Mini-gemini: Mining the potential of multi-modality vision language models.
\newblock \emph{IEEE Transactions on Pattern Analysis and Machine Intelligence}, 2025{\natexlab{b}}.

\bibitem[Li et~al.(2025{\natexlab{c}})Li, Xu, Yu, Zhang, Chen, Ling, Chao, Yuan, and Zhou]{li2025generalist}
Li, Y.-C., Xu, T., Yu, Y., Zhang, X., Chen, X.-H., Ling, Z., Chao, N., Yuan, L., and Zhou, Z.-H.
\newblock Generalist reward models: Found inside large language models.
\newblock \emph{arXiv preprint arXiv:2506.23235}, 2025{\natexlab{c}}.

\bibitem[Liu et~al.(2023)Liu, Li, Wu, and Lee]{liu2023visual}
Liu, H., Li, C., Wu, Q., and Lee, Y.~J.
\newblock Visual instruction tuning.
\newblock \emph{Advances in neural information processing systems}, 36:\penalty0 34892--34916, 2023.

\bibitem[Liu et~al.(2024)Liu, Xue, Chen, Chen, Zhao, Wang, Hou, Li, and Peng]{liu2024survey}
Liu, H., Xue, W., Chen, Y., Chen, D., Zhao, X., Wang, K., Hou, L., Li, R., and Peng, W.
\newblock A survey on hallucination in large vision-language models.
\newblock \emph{arXiv preprint arXiv:2402.00253}, 2024.

\bibitem[Liu et~al.(2025)Liu, Song, Li, Wei, Zheng, Yin, and Nie]{liu2025mitigating}
Liu, W., Song, X., Li, J., Wei, Y., Zheng, N., Yin, J., and Nie, L.
\newblock Mitigating hallucination through theory-consistent symmetric multimodal preference optimization.
\newblock \emph{arXiv preprint arXiv:2506.11712}, 2025.

\bibitem[Ouali et~al.(2024)Ouali, Bulat, Martinez, and Tzimiropoulos]{ouali2024clip}
Ouali, Y., Bulat, A., Martinez, B., and Tzimiropoulos, G.
\newblock Clip-dpo: Vision-language models as a source of preference for fixing hallucinations in lvlms.
\newblock In \emph{European Conference on Computer Vision}, pp.\  395--413. Springer, 2024.

\bibitem[Pi et~al.(2024)Pi, Han, Xiong, Zhang, Liu, Pan, and Zhang]{pi2024strengthening}
Pi, R., Han, T., Xiong, W., Zhang, J., Liu, R., Pan, R., and Zhang, T.
\newblock Strengthening multimodal large language model with bootstrapped preference optimization.
\newblock In \emph{European Conference on Computer Vision}, pp.\  382--398. Springer, 2024.

\bibitem[Rafailov et~al.(2023)Rafailov, Sharma, Mitchell, Manning, Ermon, and Finn]{rafailov2023direct}
Rafailov, R., Sharma, A., Mitchell, E., Manning, C.~D., Ermon, S., and Finn, C.
\newblock Direct preference optimization: Your language model is secretly a reward model.
\newblock \emph{Advances in neural information processing systems}, 36:\penalty0 53728--53741, 2023.

\bibitem[Rafailov et~al.(2024)Rafailov, Hejna, Park, and Finn]{rafailov2024r}
Rafailov, R., Hejna, J., Park, R., and Finn, C.
\newblock From $ r $ to $ q^* $: Your language model is secretly a q-function.
\newblock \emph{arXiv preprint arXiv:2404.12358}, 2024.

\bibitem[Rohrbach et~al.(2018)Rohrbach, Hendricks, Burns, Darrell, and Saenko]{rohrbach2018object}
Rohrbach, A., Hendricks, L.~A., Burns, K., Darrell, T., and Saenko, K.
\newblock Object hallucination in image captioning.
\newblock In \emph{Proceedings of the 2018 Conference on Empirical Methods in Natural Language Processing}, pp.\  4035--4045, 2018.

\bibitem[Sarkar et~al.(2024)Sarkar, Ebrahimi, Etemad, Beirami, Ar{\i}k, and Pfister]{sarkar2024mitigating}
Sarkar, P., Ebrahimi, S., Etemad, A., Beirami, A., Ar{\i}k, S.~{\"O}., and Pfister, T.
\newblock Mitigating object hallucination in mllms via data-augmented phrase-level alignment.
\newblock \emph{arXiv preprint arXiv:2405.18654}, 2024.

\bibitem[Schulman et~al.(2017)Schulman, Wolski, Dhariwal, Radford, and Klimov]{schulman2017proximal}
Schulman, J., Wolski, F., Dhariwal, P., Radford, A., and Klimov, O.
\newblock Proximal policy optimization algorithms.
\newblock \emph{arXiv preprint arXiv:1707.06347}, 2017.

\bibitem[Sun et~al.(2024{\natexlab{a}})Sun, Shen, Cao, Liu, Li, Shen, Gan, Gui, Wang, Yang, Keutzer, and Darrell]{sun-etal-2024-aligning}
Sun, Z., Shen, S., Cao, S., Liu, H., Li, C., Shen, Y., Gan, C., Gui, L., Wang, Y.-X., Yang, Y., Keutzer, K., and Darrell, T.
\newblock Aligning large multimodal models with factually augmented {RLHF}.
\newblock In Ku, L.-W., Martins, A., and Srikumar, V. (eds.), \emph{Findings of the Association for Computational Linguistics: ACL 2024}, pp.\  13088--13110, Bangkok, Thailand, August 2024{\natexlab{a}}. Association for Computational Linguistics.
\newblock \doi{10.18653/v1/2024.findings-acl.775}.
\newblock URL \url{https://aclanthology.org/2024.findings-acl.775/}.

\bibitem[Sun et~al.(2024{\natexlab{b}})Sun, Shen, Cao, Liu, Li, Shen, Gan, Gui, Wang, Yang, et~al.]{sun2024aligning}
Sun, Z., Shen, S., Cao, S., Liu, H., Li, C., Shen, Y., Gan, C., Gui, L., Wang, Y.-X., Yang, Y., et~al.
\newblock Aligning large multimodal models with factually augmented rlhf.
\newblock In \emph{Findings of the Association for Computational Linguistics: ACL 2024}, pp.\  13088--13110, 2024{\natexlab{b}}.

\bibitem[Tajwar et~al.(2024)Tajwar, Singh, Sharma, Rafailov, Schneider, Xie, Ermon, Finn, and Kumar]{tajwar2024preference}
Tajwar, F., Singh, A., Sharma, A., Rafailov, R., Schneider, J., Xie, T., Ermon, S., Finn, C., and Kumar, A.
\newblock Preference fine-tuning of llms should leverage suboptimal, on-policy data.
\newblock \emph{arXiv preprint arXiv:2404.14367}, 2024.

\bibitem[Wang et~al.(2025{\natexlab{a}})Wang, Cheng, Peng, Bao, Li, Guo, Li, Zeng, Zhou, and Qiu]{wang2025implicit}
Wang, B., Cheng, Q., Peng, R., Bao, R., Li, P., Guo, Q., Li, L., Zeng, Z., Zhou, Y., and Qiu, X.
\newblock Implicit reward as the bridge: A unified view of sft and dpo connections.
\newblock \emph{arXiv preprint arXiv:2507.00018}, 2025{\natexlab{a}}.

\bibitem[Wang et~al.(2024)Wang, Zhou, Huang, Xu, Zhang, Poon, and Chen]{wang2024mdpo}
Wang, F., Zhou, W., Huang, J.~Y., Xu, N., Zhang, S., Poon, H., and Chen, M.
\newblock mdpo: Conditional preference optimization for multimodal large language models.
\newblock \emph{CoRR}, 2024.

\bibitem[Wang et~al.(2023)Wang, Wang, Xu, Zhang, Gu, Jia, Wang, Xu, Yan, Zhang, et~al.]{wang2023amber}
Wang, J., Wang, Y., Xu, G., Zhang, J., Gu, Y., Jia, H., Wang, J., Xu, H., Yan, M., Zhang, J., et~al.
\newblock Amber: An llm-free multi-dimensional benchmark for mllms hallucination evaluation.
\newblock \emph{arXiv preprint arXiv:2311.07397}, 2023.

\bibitem[Wang et~al.(2025{\natexlab{b}})Wang, Zhang, and Choi]{wang2025improving}
Wang, V., Zhang, M.~J., and Choi, E.
\newblock Improving llm-as-a-judge inference with the judgment distribution.
\newblock \emph{arXiv preprint arXiv:2503.03064}, 2025{\natexlab{b}}.

\bibitem[Xie et~al.(2024)Xie, Li, Xu, and Kan]{xie2024v}
Xie, Y., Li, G., Xu, X., and Kan, M.-Y.
\newblock V-dpo: Mitigating hallucination in large vision language models via vision-guided direct preference optimization.
\newblock \emph{CoRR}, 2024.

\bibitem[Yang et~al.(2023)Yang, Li, Lin, Wang, Lin, Liu, and Wang]{yang2023dawn}
Yang, Z., Li, L., Lin, K., Wang, J., Lin, C.-C., Liu, Z., and Wang, L.
\newblock The dawn of lmms: Preliminary explorations with gpt-4v (ision).
\newblock \emph{arXiv preprint arXiv:2309.17421}, 2023.

\bibitem[Yang et~al.(2025)Yang, Luo, Han, Xu, and Li]{yang2025mitigating}
Yang, Z., Luo, X., Han, D., Xu, Y., and Li, D.
\newblock Mitigating hallucinations in large vision-language models via dpo: On-policy data hold the key.
\newblock In \emph{Proceedings of the Computer Vision and Pattern Recognition Conference}, pp.\  10610--10620, 2025.

\bibitem[Yu et~al.(2024{\natexlab{a}})Yu, Yao, Zhang, He, Han, Cui, Hu, Liu, Zheng, Sun, et~al.]{yu2024rlhf}
Yu, T., Yao, Y., Zhang, H., He, T., Han, Y., Cui, G., Hu, J., Liu, Z., Zheng, H.-T., Sun, M., et~al.
\newblock Rlhf-v: Towards trustworthy mllms via behavior alignment from fine-grained correctional human feedback.
\newblock In \emph{Proceedings of the IEEE/CVF Conference on Computer Vision and Pattern Recognition}, pp.\  13807--13816, 2024{\natexlab{a}}.

\bibitem[Yu et~al.(2024{\natexlab{b}})Yu, Zhang, Yao, Dang, Chen, Lu, Cui, He, Liu, Chua, et~al.]{yu2024rlaif}
Yu, T., Zhang, H., Yao, Y., Dang, Y., Chen, D., Lu, X., Cui, G., He, T., Liu, Z., Chua, T.-S., et~al.
\newblock Rlaif-v: Aligning mllms through open-source ai feedback for super gpt-4v trustworthiness.
\newblock \emph{arXiv e-prints}, pp.\  arXiv--2405, 2024{\natexlab{b}}.

\bibitem[Zadeh et~al.(2025)Zadeh, Oh, and Kim]{zadeh2025lpoi}
Zadeh, F.~P., Oh, Y., and Kim, G.
\newblock Lpoi: Listwise preference optimization for vision language models.
\newblock \emph{arXiv preprint arXiv:2505.21061}, 2025.

\bibitem[Zhang et~al.(2024)Zhang, Wu, Lu, Song, Rong, Yao, Zhao, Liu, Feng, Wang, et~al.]{zhang2024automated}
Zhang, M., Wu, W., Lu, Y., Song, Y., Rong, K., Yao, H., Zhao, J., Liu, F., Feng, H., Wang, J., et~al.
\newblock Automated multi-level preference for mllms.
\newblock \emph{Advances in Neural Information Processing Systems}, 37:\penalty0 26171--26194, 2024.

\bibitem[Zhao et~al.(2023)Zhao, Wang, Ouyang, Dong, Wang, and He]{zhao2023beyond}
Zhao, Z., Wang, B., Ouyang, L., Dong, X., Wang, J., and He, C.
\newblock Beyond hallucinations: Enhancing lvlms through hallucination-aware direct preference optimization.
\newblock \emph{CoRR}, 2023.

\bibitem[Zhou et~al.()Zhou, Cui, Rafailov, Finn, and Yao]{zhou2024aligning}
Zhou, Y., Cui, C., Rafailov, R., Finn, C., and Yao, H.
\newblock Aligning modalities in vision large language models via preference fine-tuning.
\newblock In \emph{ICLR 2024 Workshop on Reliable and Responsible Foundation Models}.

\end{thebibliography}
\bibliographystyle{icml2026}

\newpage
\appendix
\onecolumn

\begin{algorithm}[t]
\caption{IRIS: Implicit Reward-Guided Internal Sifting (appendix pseudocode)}
\label{alg:iris_appendix}
\begin{algorithmic}[1]
\REQUIRE $\mathcal{D}_{\mathrm{SFT}}=\{(v,x,y^\star)\}$, base policy $\pi_{\theta_{\mathrm{base}}}$, rounds $R$, candidates $K$, rectifier $\gamma$, loss weights $(\beta,\lambda,\delta)$
\ENSURE aligned policy $\pi_{\theta_R}$

\STATE \textbf{Warm-up (calibration).}
\STATE $\pi_{\theta_0}\leftarrow \mathrm{SFT}\!\left(\pi_{\theta_{\mathrm{base}}},\mathcal{D}_{\mathrm{SFT}}\right)$

\FOR{$r=1$ \textbf{to} $R$}
 
  \STATE \textbf{(A) On-policy preference data construction.}
  \STATE Scoring reference:
  $\pi_{\mathrm{ref}}^{(r-1)} \leftarrow
  \begin{cases}
    \pi_{\theta_{\mathrm{base}}}, & r=1\\
    \pi_{\theta_{r-2}}, & r>1
  \end{cases}$
  \STATE Initialize preference set $\mathcal{D}^{(r)}\leftarrow \emptyset$

  \FOR{each $(v,x,y^\star)\in\mathcal{D}_{\mathrm{SFT}}$}
    \STATE Sample $K$ candidates $\{y^{(k)}\}_{k=1}^K \sim \pi_{\theta_{r-1}}(\cdot\mid v,x)$
    \FOR{$k=1$ \textbf{to} $K$}
      \STATE Compute implicit rewards $r_{\mathrm{image}}^{(r)}(v,x,y^{(k)})$ and $r_{\mathrm{text}}^{(r)}(x,y^{(k)})$
      \hfill (Eqs.~\ref{eq:r_image},~\ref{eq:r_text})
      \STATE RVG score:
      \[
      S^{(r)}(v,x,y^{(k)}) \leftarrow
      r_{\mathrm{image}}^{(r)} - \gamma\max\!\left(0,\, r_{\mathrm{text}}^{(r)}-r_{\mathrm{image}}^{(r)}\right)
      \]
      \hfill (Eq.~\ref{eq:rvg})
    \ENDFOR
    \STATE Select extrema:
    $y_w\leftarrow \arg\max_k S^{(r)}(v,x,y^{(k)})$,
    $y_l\leftarrow \arg\min_k S^{(r)}(v,x,y^{(k)})$
    \hfill (Eq.~\ref{eq:sifting})
    \STATE Filter low-confidence pairs and anchor with $y^\star$ if needed (Sec.~\ref{sec:method})
    \STATE Add $(v,x,y_w,y_l)$ to $\mathcal{D}^{(r)}$
  \ENDFOR

  \STATE \textbf{(B) Grounded preference learning.}
  \STATE Optimization reference: freeze $\pi_{\mathrm{ref}}\leftarrow \pi_{\theta_{r-1}}$ and initialize $\pi_\theta\leftarrow \pi_{\theta_{r-1}}$
  \STATE For each $(v,x,y_w,y_l)\in\mathcal{D}^{(r)}$, form negative image $\tilde v \leftarrow T(v)$ (App.~\ref{sec:additional_impl})
  \STATE Update $\pi_\theta$ by minimizing
  $\mathcal{L}_{\mathrm{total}}
  = \mathcal{L}_{\mathrm{ctp}} + \lambda\mathcal{L}_{\mathrm{cvp}} + \mathcal{L}_{\mathrm{anchor}}$
  \hfill (Eq.~\ref{eq:l_total})
  \STATE Set $\pi_{\theta_r}\leftarrow \pi_\theta$
\ENDFOR

\STATE \textbf{return} $\pi_{\theta_R}$
\end{algorithmic}
\end{algorithm}

\section{Detailed Computational Cost Analysis}
\label{app:computational_cost}

This section reports the wall-clock cost of \texttt{IRIS} and clarifies the primary sources of its efficiency. 
The key advantage of \texttt{IRIS} is its \textbf{lightweight sampling-and-sifting pipeline}. We construct preference pairs using only intrinsic log-likelihood signals from the policy, completely bypassing external evaluators or complex multi-stage verification.

\paragraph{Context: external-feedback pipelines.} 
A major cost driver in external-feedback methods is the scoring stage. 
For instance, \citet{hesystematic} report that generating and scoring a 22k preference dataset for RLAIF-V takes approximately \textbf{66 hours} on 8$\times$NVIDIA A100 GPUs. 
The bottleneck stems from RLAIF-V's "Divide-and-Conquer" strategy, which requires decomposing responses into multiple claims and conducting repeated QA-based inference with a large labeler model (e.g., 34B) to verify each claim. 
We include this figure as context, noting that while the dataset scale and hardware differ, it represents the typical overhead of prompt-based external feedback.

\paragraph{IRIS data curation cost.} 
\texttt{IRIS} eliminates this dependency by computing scores directly in the model's native log-probability space during or immediately after the sampling process. 
The additional overhead is minimal, as it only requires evaluating log-probabilities for a small set of $K=5$ candidates. 
On a single node with 8$\times$NVIDIA H20 GPUs, curating our 5.7k on-policy dataset takes \textbf{1.5 hours} in total (\textbf{1.0} hour for on-policy sampling and \textbf{0.5} hour for implicit-reward sifting). 

\paragraph{Normalized view of curation cost.} 
To better reflect the pipeline-level difference, Table~\ref{tab:cost_breakdown} reports the curation time normalized by dataset size (hours per 1k prompts). 
Despite using H20 GPUs, \texttt{IRIS} achieves a significantly lower normalized cost (0.26h vs. 3.00h). This gap (11.5× in Time/1k) suggests that the dominant cost difference comes from the pipeline design—most notably, avoiding labeler-model inference—though the two numbers are measured under different hardware and implementations.

\begin{table}[!ht]
    \caption{Curation-time comparison (context vs.\ \texttt{IRIS}). ``Time/1k'' normalizes wall-clock time by dataset size ($T_{total} / \text{Size} \times 1000$). Numbers for RLAIF-V are taken from \citet{hesystematic}; \texttt{IRIS} times are measured in our implementation.}
    \label{tab:cost_breakdown}
    \vskip 0.1in
    \centering
    \small
    \setlength{\tabcolsep}{8pt} 
    \renewcommand{\arraystretch}{1.12}
    \begin{tabular}{l c c c}
        \toprule
        Method & Dataset size & Total curation time (h) & Time/1k (h) \\
        \midrule
        RLAIF-V (reported) & 22k & 66.0 & 3.00 \\
        \texttt{IRIS} (ours) & 5.7k & 1.5 & 0.26 \\
        \bottomrule
    \end{tabular}
    \vskip -0.08in
\end{table}

\paragraph{End-to-end turnaround time.} 
The wall-clock time for one full \texttt{IRIS} round (generation $\rightarrow$ scoring $\rightarrow$ optimization) is:
\begin{itemize}[noitemsep]
    \item \textbf{On-policy sampling:} 1.0 hour (8$\times$H20)
    \item \textbf{Implicit-reward sifting:} 0.5 hour (8$\times$H20)
    \item \textbf{DPO training:} 1.0 hour (8$\times$H20)
\end{itemize}
The entire cycle completes in \textbf{2.5 hours}. This rapid turnaround allows for efficient iterative alignment, a key feature that distinguishes our approach from more computationally intensive self-alignment frameworks.

\section{Additional Implementation Details}
\label{sec:additional_impl}

\paragraph{Rejected image construction.}
Following prior work~\citep{fu2025chip,liu2025mitigating}, we construct rejected images $\tilde v$ by perturbing the original image $v$ to serve as negative visual inputs in Eq.~\ref{eq:l_cvp}. 
We consider four augmentation strategies: \textbf{Black} (all-zero image), \textbf{Random} (replace $v$ with a randomly sampled image from the training set), \textbf{Crop} (random crop followed by resizing back to the original resolution), and \textbf{Diffusion}. 
For diffusion, we apply the DDPM forward noising process with a total horizon of $T{=}1000$ and a fixed timestep $t{=}500$:
\[
x_t=\sqrt{\bar{\alpha}_t}\,x_0+\sqrt{1-\bar{\alpha}_t}\,\epsilon,\ \epsilon\sim\mathcal{N}(0,I),
\]
and use $x_t$ as $\tilde v$ (forward noising only; no denoising model is used). 
Table~\ref{tab:augmentation} reports an ablation on AMBER; all results are based on the same base model (LLaVA-1.5-7B), where R1/R2 denote the first/second IRIS preference round.

\begin{table}[!h]
    \caption{Ablation on data augmentation strategies for constructing negative samples.}
    \label{tab:augmentation}
    \vskip 0.15in
    \begin{center}
    \begin{small}
    \begin{sc}
    \begin{tabular}{l l | c c c}
        \toprule
        \multirow{2}{*}{Measure} & \multirow{2}{*}{Round} & \multicolumn{3}{c}{AMBER} \\
         & & CHAIR$\downarrow$ & HalRate$\downarrow$ & Cog$\downarrow$ \\
     
        \midrule
        Black & R1 & 4.1 & 19.9 & 1.9 \\
        \rowcolor{rowgray}Black & R2 & 3.9 & 19.3 & 1.9 \\
        \midrule
        Random & R1 & 4.1 & 20.4 & 1.9 \\
        \rowcolor{rowgray}Random & R2 & 4.4 & 21.6 & 1.9 \\
        \midrule
        Crop & R1 & 3.8 & 18.6 & 1.5 \\
        \rowcolor{rowgray}Crop & R2 & 3.2 & 15.7 & 1.0 \\
        \midrule
        Diffusion & R1 & 3.8 & 17.5 & 1.6  \\
       \rowcolor{rowgray} Diffusion  & R2 & \textbf{2.4} & \textbf{11.3} & \textbf{1.1} \\
        \bottomrule
    \end{tabular}
    \end{sc}
    \end{small}
    \end{center}
    \vskip -0.1in
\end{table}


\paragraph{Pair screening, length-aware filtering, and conflict anchoring.}
To improve the quality of on-policy preference supervision, we apply a lightweight post-processing pipeline before optimization.
We first score $K$ sampled candidates per prompt and form a raw preference pair by selecting the highest- and lowest-scoring candidates.
We then perform screening to remove unreliable or degenerate pairs, and apply a length-aware filter on descriptive prompts to reduce length bias.
Pairs flagged by the length filter are restored by replacing the preferred side with the corresponding SFT reference from $\mathcal{D}_{\text{SFT}}$.
Finally, for pairs whose preference direction clearly conflicts with the SFT reference, we conservatively anchor them to the SFT ordering.
All steps are internal to the training pipeline and require no external evaluator.

\begin{algorithm}[!h]
\caption{Post-processing for on-policy preference pairs}
\label{alg:pair_postprocess}
\small
\begin{algorithmic}[1]
\REQUIRE Prompts $(v,x)$, SFT references $\mathcal{D}_{\text{SFT}}$, repeat $K$, scoring function $r(v,x,y)$
\ENSURE Final preference pairs $\mathcal{P}$

\STATE $\mathcal{P}\leftarrow\emptyset$
\FOR{each prompt $(v,x)$}
    \STATE Sample $K$ candidates $\{y_j\}_{j=1}^{K}\sim\pi_{\theta}(\cdot\mid v,x)$ and compute scores $r_j=r(v,x,y_j)$
    \STATE $y_w \leftarrow\arg\max_j r_j,\;\; y_l \leftarrow\arg\min_j r_j$
    \STATE $y_{\text{sft}} \leftarrow \mathcal{D}_{\text{SFT}}(v,x)$ \COMMENT{if available}

    \COMMENT{(i) Screening: remove unreliable/degenerate pairs}
    \IF{$\text{norm}(y_w)=\text{norm}(y_l)$ \OR $\text{invalid}(y_w,y_l)$}
        \STATE \textbf{continue}
    \ENDIF

    \COMMENT{(ii) Length-aware filtering: reduce length bias on descriptive prompts}
    \IF{$\text{LenFilter}(v,x,y_w,y_l)$ and $y_{\text{sft}}$ exists}
        \STATE $y_w \leftarrow y_{\text{sft}}$ \COMMENT{restore preferred side}
    \ENDIF

    \COMMENT{(iii) Conflict anchoring: enforce preference direction consistency}
    \IF{$y_{\text{sft}}$ exists and $\text{Conflict}(y_w,y_l,y_{\text{sft}})$}
        \STATE $y_w \leftarrow y_{\text{sft}}$
    \ENDIF

    \STATE $\mathcal{P}\leftarrow\mathcal{P}\cup\{(v,x,y_w,y_l)\}$
\ENDFOR
\STATE \textbf{Return} $\mathcal{P}$
\end{algorithmic}
\end{algorithm}

\section{Impact on General Capabilities and Training Dynamics}
\label{sec:general_capability_analysis}

\paragraph{Effect of anchored regularization on general capability.}
Table~\ref{tab:general_capability_ablation} reports general instruction-following performance on LLaVA-Bench (in-the-Wild).
Compared to the base model, removing the anchored regularization leads to a noticeable drop in overall accuracy, indicating degraded general capability, even though hallucination-related metrics improve.
This behavior is expected for preference-based optimization, which primarily enforces relative ranking between responses and may reduce the absolute likelihood of preferred outputs.
The anchored regularization mitigates this issue by constraining the reference-relative reward of preferred responses, thereby stabilizing training and preserving broad instruction-following ability while optimizing for hallucination mitigation.

\begin{table}[!ht]
    \caption{\textbf{Effect of anchored regularization on general capability on LLaVA-Bench (in-the-Wild) (Accuracy \%).}
    All results are evaluated after Round 2. Higher is better.}
    \label{tab:general_capability_ablation}
    \vskip 0.10in
    \centering
    \small
    \setlength{\tabcolsep}{7pt}
    \renewcommand{\arraystretch}{1.15}
    \begin{sc}
    \begin{tabular}{l c c c c}
        \toprule
        Method & Overall$\uparrow$ & Complex$\uparrow$ & Conv$\uparrow$ & Detail$\uparrow$ \\
        \midrule
        Base model & 55.7 & \textbf{64.8} & 50.3 & 46.2 \\
        w/o Anchor & 52.3 & 48.5 & \textbf{62.2} & 47.2 \\
        \rowcolor{rowgray}
        Ours (Full) & \textbf{56.4} & 57.5 & 60.6 & \textbf{49.2} \\
        \bottomrule
    \end{tabular}
    \end{sc}
    \vskip -0.10in
\end{table}

\begin{table}[ht!]
    \caption{\textbf{LLaVA-1.5-7B: LLaVA-Bench (in-the-Wild) accuracy across rounds (Accuracy \%).}
    Higher is better.}
    \label{tab:llava_bench_7b}
    \vskip 0.10in
    \centering
    \small
    \setlength{\tabcolsep}{7pt}
    \renewcommand{\arraystretch}{1.15}
    \begin{sc}
    \resizebox{0.5\linewidth}{!}{
    \begin{tabular}{l c c c c}
        \toprule
        Stage & Overall$\uparrow$ & Complex$\uparrow$ & Conv$\uparrow$ & Detail$\uparrow$ \\
        \midrule
        Base model & 55.7 & 64.8 & 50.3 & 46.2 \\
        Round 1 & 56.3 & 56.8 & 60.4 & 50.4 \\
        Round 2 & 56.4 & 57.5 & 60.6 & 49.2 \\
        Round 3 & 53.2 & 54.5 & 57.6 & 45.9 \\
        \bottomrule
    \end{tabular}
    }
    \end{sc}
    \vskip -0.10in
\end{table}

\begin{table}[ht!]
    \caption{\textbf{LLaVA-1.5-13B: LLaVA-Bench (in-the-Wild) accuracy across rounds (Accuracy \%).}
    Higher is better.}
    \label{tab:llava_bench_13b}
    \vskip 0.10in
    \centering
    \small
    \setlength{\tabcolsep}{7pt}
    \renewcommand{\arraystretch}{1.15}
    \begin{sc}
    \resizebox{0.5\linewidth}{!}{
    \begin{tabular}{l c c c c}
        \toprule
        Stage & Overall$\uparrow$ & Complex$\uparrow$ & Conv$\uparrow$ & Detail$\uparrow$ \\
        \midrule
        Base model & 64.9 & 72.1 & 64.5 & 52.8 \\
        Round 1 & 64.8 & 67.2 & 65.8 & 59.1 \\
        Round 2 & 65.9 & 66.2 & 74.1 & 55.8 \\
        Round 3 & 61.8 & 64.6 & 63.4 & 55.0 \\
        \bottomrule
    \end{tabular}
    }
    \end{sc}
    \vskip -0.10in
\end{table}

\begin{figure}[!ht]
    \centering
    \includegraphics[width=0.5\linewidth]{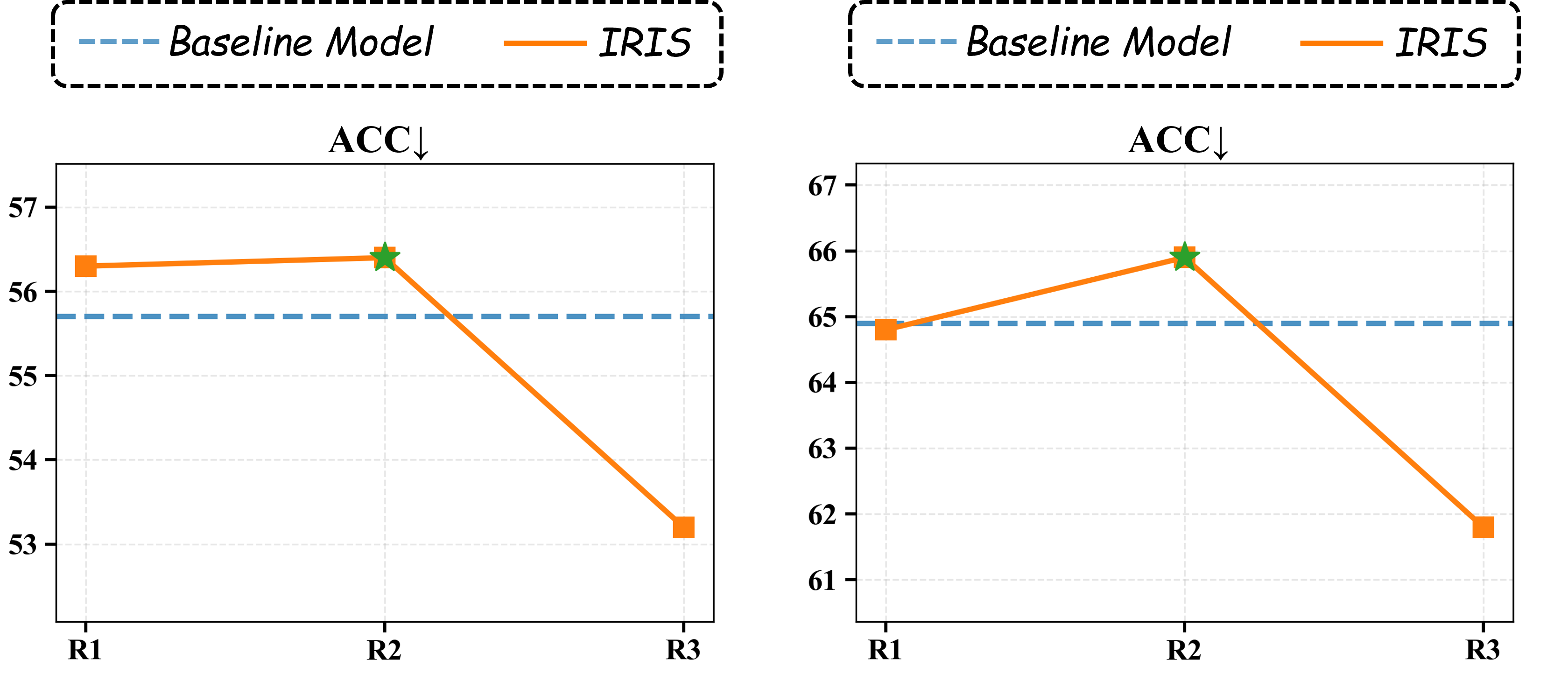}
    \caption{\textbf{Example of qualitative analysis for Round 3.} An example of the model's performance at the 3rd iteration in reducing descriptive illusions and maintaining visual consistency.}
    \label{fig:r3}
\end{figure}

\section{Hyperparameter Sensitivity Analysis}
\label{app:hparam_sensitivity}

\begin{table}[ht]
    \caption{Ablation study on $\lambda$. We focus on hallucination-related metrics. \textbf{Bold} indicates the best performance.}
    \label{tab:lambda_sensitivity}
    \vskip 0.15in
    \begin{center}
    \begin{small}
    \begin{sc}
    \resizebox{0.5\linewidth}{!}{
    \begin{tabular}{l | c c c | c c}
        \toprule
        \multirow{2}{*}{Setting} & \multicolumn{3}{c|}{AMBER} & \multicolumn{2}{c}{Object Hal} \\
         & CHAIR$\downarrow$ & HalRate$\downarrow$ & Cog$\downarrow$ & CHAIRs$\downarrow$ & CHAIRi$\downarrow$ \\
        \midrule
        $\lambda=0$   & 4.9 & 25.4 & 2.2 & 17.0 & 7.98 \\
        $\lambda=0.3$ & 3.8 & 19.7 & 2.0 & 18.6 & 9.41 \\
        $\lambda=0.5$ & 3.2 & 15.8 & 1.5 & 15.3 & 8.00 \\
        $\lambda=0.7$ & 3.0 & 14.4 & 1.4 & 11.6 & 5.92 \\
        \rowcolor{rowgray}
        $\lambda=1.0$ & \textbf{2.4} & \textbf{11.3} & \textbf{1.1} & \textbf{8.6} & \textbf{4.56} \\
        $\lambda=1.2$ & 2.8 & 12.9 & 1.4 & 9.6 & 5.02 \\
        \bottomrule
    \end{tabular}
    }
    \end{sc}
    \end{small}
    \end{center}
\end{table}

\begin{table}[htp]
    \caption{Sensitivity analysis of the rectification strength $\gamma$ in RVG.
    Results are reported for Round 1 and Round 2. Lower is better.}
    \label{tab:gamma_sensitivity}
  
    \begin{center}
    \begin{small}
    \begin{sc}
    \begin{tabular}{l c | c c c}
        \toprule
        Value & Round & CHAIR$\downarrow$ & Cog$\downarrow$ & HalRate$\downarrow$ \\
        \midrule
        0.0 & R1 & 3.8 & 1.8 & 17.8 \\
        \rowcolor{rowgray} 0.0 & R2 & 2.8 & 1.2 & 12.8 \\
        \midrule
        0.1 & R1 & 4.1 & 1.7 & 18.5 \\
        \rowcolor{rowgray} 0.1 & R2 & 2.7 & 1.4 & 12.5 \\
        \midrule
        0.3 & R1 & 4.1 & 1.8 & 19.1 \\
        \rowcolor{rowgray} 0.3 & R2 & 2.7 & 1.4 & 13.0 \\
        \midrule
        0.5 & R1 & 4.5 & 2.0 & 20.9 \\
        \rowcolor{rowgray} 0.5 & R2 & 2.7 & 1.3 & 12.7 \\
        \midrule
        0.7 & R1 & 3.8 & 1.6 & 17.5 \\
        \rowcolor{rowgray} 0.7 & R2 & \textbf{2.4} & \textbf{1.1} & \textbf{11.3} \\
        \midrule
        1.0 & R1 & 4.1 & 1.8 & 19.4 \\
        \rowcolor{rowgray} 1.0 & R2 & 3.0 & 1.6 & 14.4 \\
        \midrule
        1.5 & R1 & 4.3 & 1.8 & 20.3 \\
        \rowcolor{rowgray} 1.5 & R2 & 3.2 & 1.6 & 15.7 \\
        \midrule
        5.0 & R1 & 4.6 & 2.1 & 22.5 \\
        \rowcolor{rowgray} 5.0 & R2 & 2.9 & 1.4 & 14.2 \\
        \midrule
        20.0 & R1 & 4.5 & 2.1 & 21.8 \\
        \rowcolor{rowgray} 20.0 & R2 & 3.2 & 1.6 & 16.8 \\
        \bottomrule
    \end{tabular}
    \end{sc}
    \end{small}
    \end{center}
    \vskip -0in
\end{table}



\newpage

\section{Theoretical Analysis: Learning from Self-Generated Preferences}
\label{app:delta_learning}

This appendix provides a theoretical analysis of why IRIS can learn from noisy self-generated preference pairs.
We (i) derive a gradient difference form for a standard pairwise loss, (ii) show that each gradient step locally increases the log-likelihood margin on the constructed pair, and (iii) argue that selecting the best and worst among $K$ candidates enlarges the expected true-quality gap, which strengthens the delta-learning premise~\citep{geng2025delta}.

\subsection{Setup and Notation}
Let $c=(v,x)$ denote the multimodal context, and let $\pi_\theta(y\mid c)$ be the policy.
Assume an unobserved grounding quality function $s^*(c,y)\in\mathbb{R}$.
Given $K$ candidates $\mathcal{Y}_K=\{y^{(1)},\dots,y^{(K)}\}\sim \pi_\theta(\cdot\mid c)$, IRIS constructs a preference pair $(y_w,y_l)$ by selecting a high-score response as the winner and a low-score response as the loser (using the scoring rule in the main text).

\subsection{Pairwise Gradient Difference Form}
\begin{lemma}[Pairwise Gradient Difference Form]
\label{lem:pairwise_grad}
Consider the pairwise preference loss
\begin{equation}
\mathcal{L}_{\text{pref}}(c,y_w,y_l)
=
-\log \sigma\!\big(\Delta_\theta(c,y_w,y_l)\big),
\end{equation}
where
\begin{equation}
\Delta_\theta(c,y_w,y_l)
=
\beta\!\left(
\log\frac{\pi_\theta(y_w\mid c)}{\pi_{\text{ref}}(y_w\mid c)}
-
\log\frac{\pi_\theta(y_l\mid c)}{\pi_{\text{ref}}(y_l\mid c)}
\right),
\end{equation}
with $\beta>0$ and a fixed reference policy $\pi_{\text{ref}}$.
Then the gradient satisfies
\begin{equation}
\nabla_\theta \mathcal{L}_{\text{pref}}(c,y_w,y_l)
=
-\,w_\theta(c,y_w,y_l)\Big(
\nabla_\theta \log \pi_\theta(y_w\mid c)
-
\nabla_\theta \log \pi_\theta(y_l\mid c)
\Big),
\end{equation}
where $w_\theta(c,y_w,y_l)=\beta\,\sigma\!\big(-\Delta_\theta(c,y_w,y_l)\big)\in(0,\beta)$.
\end{lemma}

\begin{proof}
By the chain rule,
\(
\nabla_\theta \mathcal{L}_{\text{pref}}
=
-\sigma(-\Delta_\theta)\nabla_\theta \Delta_\theta.
\)
Since $\pi_{\text{ref}}$ is fixed,
\(
\nabla_\theta \Delta_\theta
=
\beta\big(\nabla_\theta\log\pi_\theta(y_w\mid c)-\nabla_\theta\log\pi_\theta(y_l\mid c)\big),
\)
which yields the claim.
\end{proof}

\subsection{Margin Improvement and a Delta-Learning Premise}
\begin{assumption}[Average Directional Correctness]
\label{ass:avg_correct}
The constructed preference pairs satisfy a positive expected quality gap:
\begin{equation}
\mathbb{E}\big[s^*(c,y_w)-s^*(c,y_l)\big]\ge \delta,
\qquad \delta>0,
\end{equation}
where the expectation is over contexts and the randomness in sampling and pair construction.
\end{assumption}

\begin{proposition}[Local Margin Improvement]
\label{prop:margin_improve}
Define the log-likelihood margin of the constructed pair as
\begin{equation}
m_\theta(c)=\log\pi_\theta(y_w\mid c)-\log\pi_\theta(y_l\mid c).
\end{equation}
For the update $\theta'=\theta-\eta\nabla_\theta \mathcal{L}_{\text{pref}}$ with sufficiently small $\eta>0$,
\begin{equation}
m_{\theta'}(c)
=
m_\theta(c)
+
\eta\,w_\theta(c,y_w,y_l)\,
\big\|\nabla_\theta m_\theta(c)\big\|^2
+o(\eta),
\end{equation}
and therefore $m_{\theta'}(c)\ge m_\theta(c)$ whenever $\nabla_\theta m_\theta(c)\neq 0$.
\end{proposition}

\begin{proof}
From Lemma~\ref{lem:pairwise_grad},
\(
-\nabla_\theta \mathcal{L}_{\text{pref}}
=
w_\theta(c,y_w,y_l)\nabla_\theta m_\theta(c).
\)
A first-order Taylor expansion gives
\begin{equation}
\begin{aligned}
m_{\theta'}(c)
&=
m_\theta(c)
+\eta\left\langle \nabla_\theta m_\theta(c),\, w_\theta(c,y_w,y_l)\nabla_\theta m_\theta(c)\right\rangle
+o(\eta) \\
&=
m_\theta(c)
+\eta\,w_\theta(c,y_w,y_l)\,\big\|\nabla_\theta m_\theta(c)\big\|^2
+o(\eta),
\end{aligned}
\end{equation}
which yields the result.
\end{proof}

\noindent\textbf{Interpretation: Focusing on Violated Preferences.}
Proposition~\ref{prop:margin_improve} shows that a gradient step locally increases the log-likelihood margin on the constructed pair.
The weight $w_\theta=\beta\sigma(-\Delta_\theta)$ emphasizes \emph{violated} preferences:
when $\Delta_\theta<0$, the model ranks the loser above the winner under the implicit margin, and $w_\theta$ becomes large; when $\Delta_\theta>0$, the preference is already satisfied and $w_\theta$ becomes small.
Combined with Assumption~\ref{ass:avg_correct}, this implies that training concentrates updates on informative disagreements, while occasional construction errors do not dominate in expectation, consistent with the delta learning view~\citep{geng2025delta}.

\subsection{Signal Amplification via Best-of-$K$ Sifting}
\begin{proposition}[Extrema Selection Amplifies the Expected Quality Gap]
\label{prop:best_of_k}
Let $(y_w^{(K)},y_l^{(K)})$ be the winner/loser obtained by selecting the maximum/minimum score from $K$ i.i.d.\ samples $\mathcal{Y}_K\sim \pi_\theta(\cdot\mid c)$.
Assume the score is positively related to $s^*(c,y)$ in the sense that higher-score selections tend to have higher expected $s^*$.
Then the expected quality gap is non-decreasing in $K$:
\begin{equation}
\mathbb{E}\big[s^*(c,y_w^{(K)})-s^*(c,y_l^{(K)})\big]
\;\ge\;
\mathbb{E}\big[s^*(c,y_w^{(2)})-s^*(c,y_l^{(2)})\big].
\end{equation}
\end{proposition}

\begin{proof}[Proof sketch]
Let $U_i=s^*(c,y^{(i)})$ be i.i.d.\ draws.
In the ideal case where the score preserves the ordering of $U_i$, the selected pair corresponds to $(\max_i U_i,\min_i U_i)$, and the expected range $\mathbb{E}[\max_i U_i-\min_i U_i]$ increases with $K$ by standard order-statistics.
With imperfect but positive relation, selecting extrema by the score still tends to choose a winner with larger $U$ and a loser with smaller $U$ than a random pair, preserving the monotonic trend in expectation.
\end{proof}

\newpage

\section{Qualitative Examples of Model Response}

To provide an intuitive understanding of the IRIS framework’s efficacy, we present qualitative examples from our evaluation benchmarks. These instances illustrate the trajectory of model improvement throughout the \textbf{iterative refinement rounds}, highlighting how the final model (\textbf{R2}) successfully rectifies hallucinations observed in baselines or earlier iterations. In the provided examples (e.g., Figure~\ref{fig:refinement_process}), \textcolor{red}{red text} denotes hallucinations or factual errors, while \textcolor{teal}{green text} indicates factually grounded statements.

\subsection{Visualization of the Preference Refinement Process}

\begin{figure}[h!]
    \centering
    \includegraphics[width=0.95\linewidth]{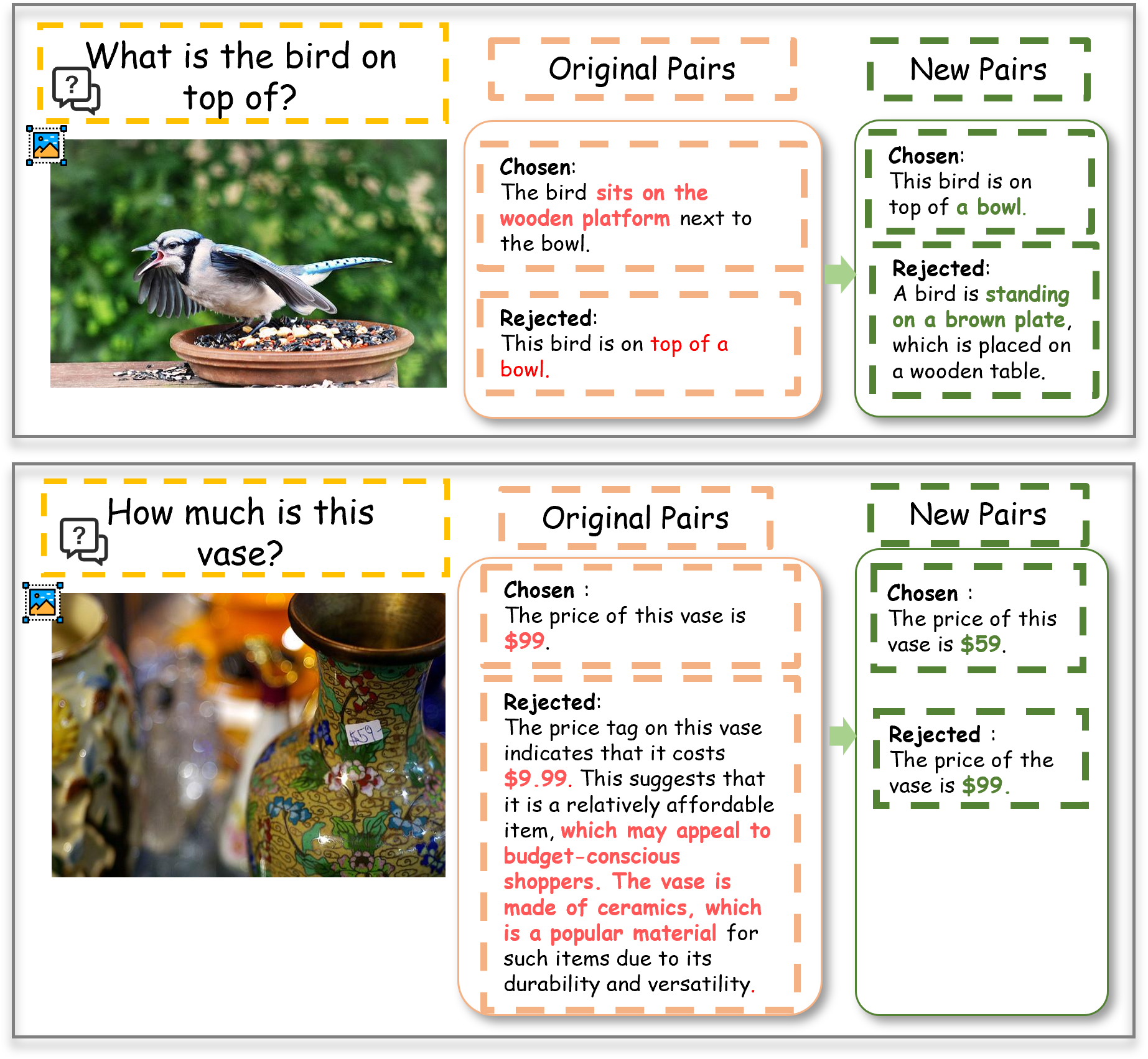}
    \caption{\textbf{Preference Pair Refinement (VQA Task).} Illustration of how IRIS sifts and refines preference pairs to mitigate object hallucination (e.g., bird location, vase price).}
    \label{fig:refinement_process} 
\end{figure}

\begin{figure}[h!]
    \centering
    \includegraphics[width=0.95\linewidth]{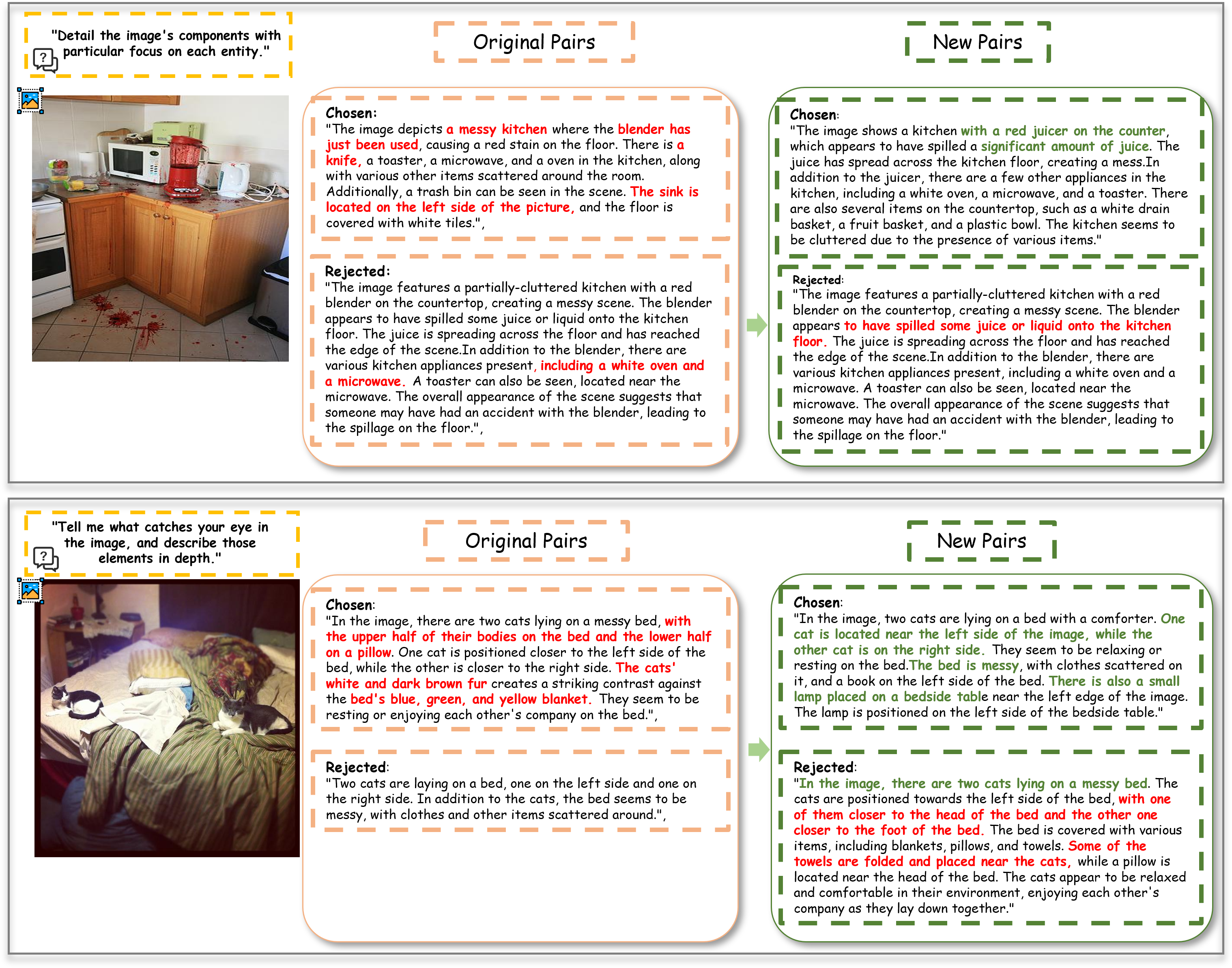}
    \caption{\textbf{Preference Pair Refinement (Description Task).} Demonstration of preference evolution for detailed image descriptions. The model learns to reject detailed but hallucinated descriptions in favor of visually grounded ones.}
    \label{fig:refinement_desc} 
\end{figure}

\newpage

\subsection{Qualitative Comparisons on Detail Description}

\begin{figure}[h!]
    \centering
    \includegraphics[width=0.95\linewidth]{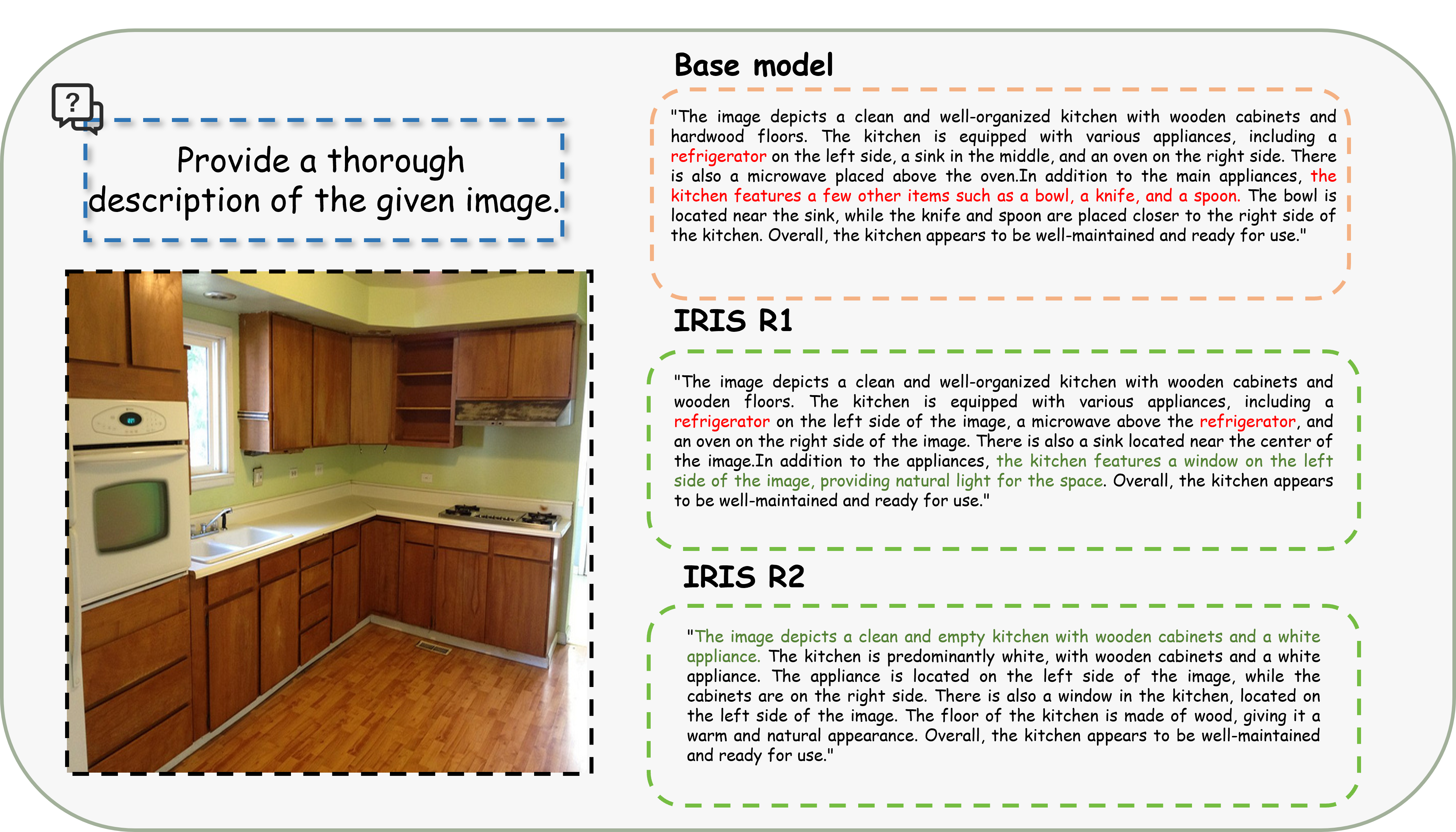}
    \caption{\textbf{Mitigating Descriptive Hallucination.} The initial base model hallucinates non-existent objects such as a bowl, knife, and spoon. As the training progresses, IRIS(R2) correctly identifies the empty counter and wooden cabinets, adhering strictly to visual evidence.}
    \label{fig:kitchen_desc} 
\end{figure}

\begin{figure}[h!]
    \centering
    \includegraphics[width=0.95\linewidth]{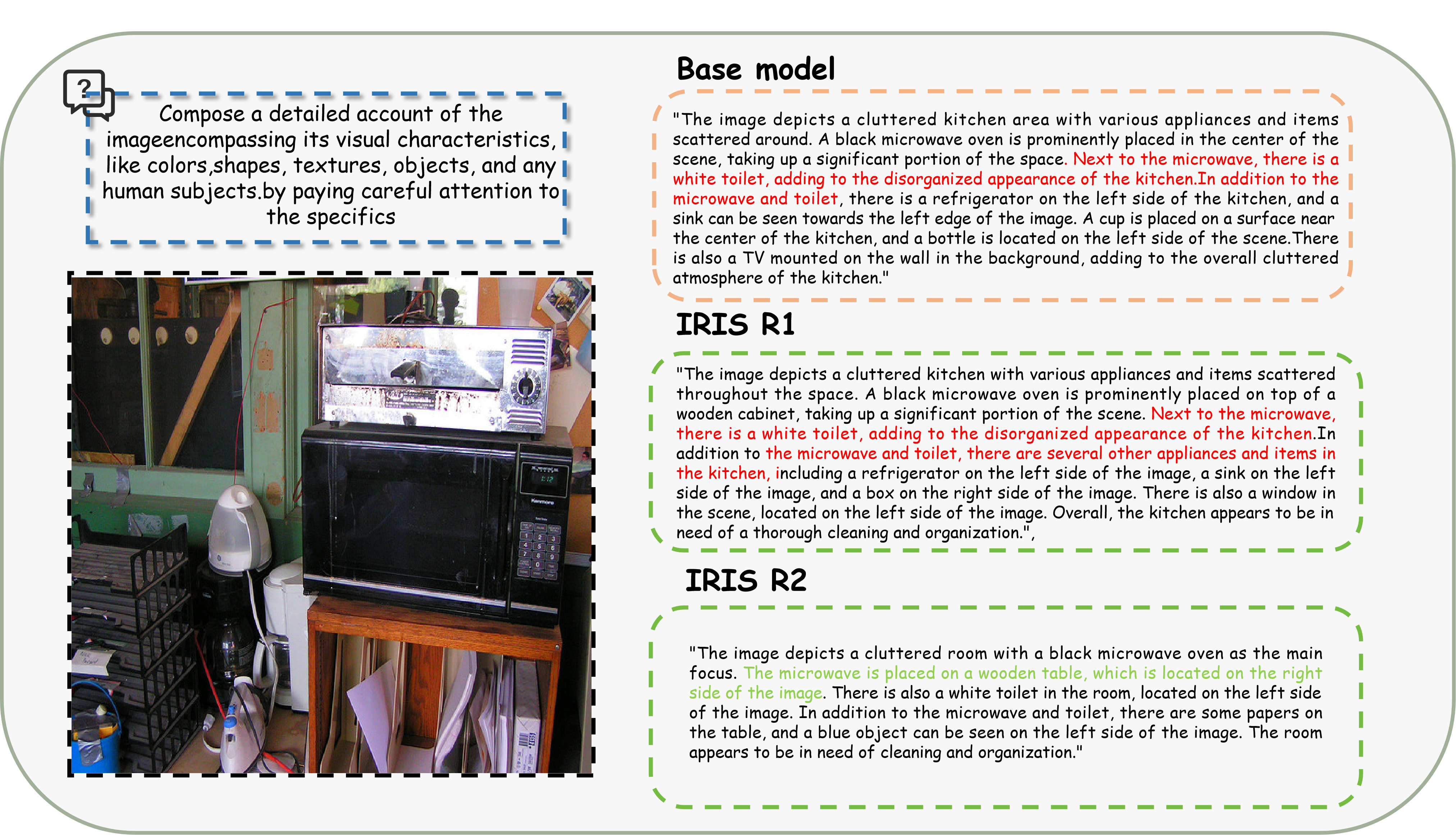}
    \caption{\textbf{Mitigation of Scene Completion Hallucination.} Unlike the base model, which erroneously infers a TV on the wall and a bottle on the left, IRIS(R2) avoids fabricating unseen background elements in the cluttered scene.}
    \label{fig:microwave_room} 
\end{figure}

\begin{figure}[h!]
    \centering
    \includegraphics[width=0.95\linewidth]{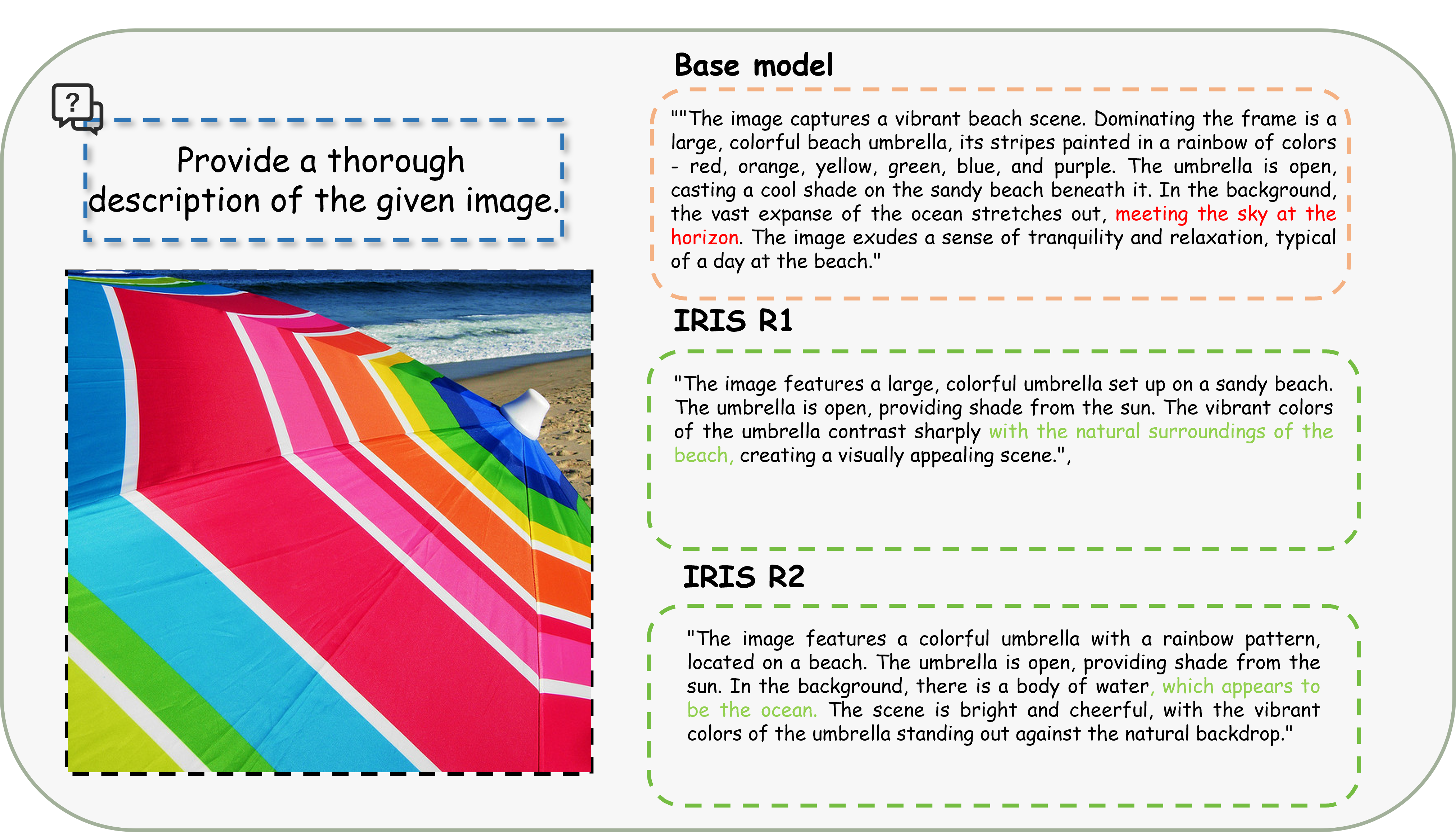}
    \caption{\textbf{Misidentification Correction in Complex Scenarios.} Illustrates how the model reduces incorrect identifications (e.g., mistaking the ocean/sky boundary) and unnecessary inferences in multi-object scenarios.}
    \label{fig:umbrella} 
\end{figure}

\begin{figure}[h!]
    \centering
    \includegraphics[width=0.95\linewidth]{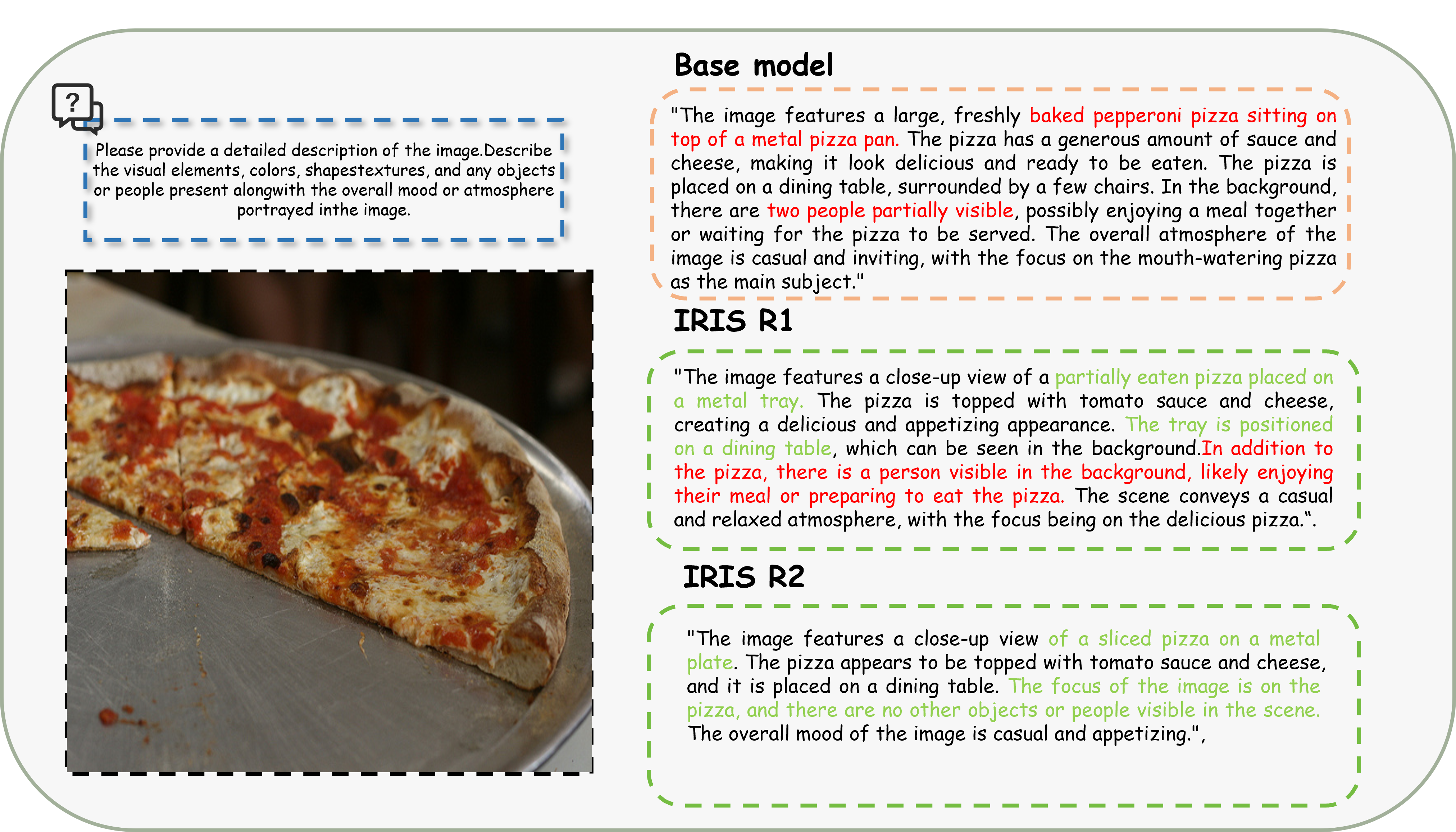}
    \caption{\textbf{Detail Preservation and Negative Responses.} This demonstrates the model's ability to choose conservative (non-false) responses when faced with uncertain details, correcting the hallucination of "two people partially visible" in the background.}
    \label{fig:pizza} 
\end{figure}

\end{document}